%% file: paper.tex
\documentclass[10pt]{article} 
\usepackage[accepted]{rlc}
\input{math_commands.tex}

\usepackage[utf8]{inputenc} 
\usepackage[T1]{fontenc}    
\usepackage{hyperref}       
\usepackage{url}            
\usepackage{booktabs}       
\usepackage{amsfonts}       
\usepackage{nicefrac}       
\usepackage{microtype}      
\usepackage{xcolor}         

\usepackage{caption}
\usepackage{subcaption}
\usepackage{algorithmic}
\usepackage{algorithm}
\usepackage{wrapfig}

\usepackage{amssymb}
\usepackage{mathtools}
\usepackage{amsthm}
\usepackage{bbm}
\usepackage{dsfont}

\theoremstyle{plain}
\newtheorem{theorem}{Theorem}[section]
\newtheorem{proposition}[theorem]{Proposition}

\newtheorem{corollary}[theorem]{Corollary}
\theoremstyle{definition}
\newtheorem{definition}[theorem]{Definition}

\theoremstyle{remark}

\newcommand{\St}{\mathcal{S}}
\newcommand{\A}{\mathcal{A}}

\newcommand{\False}{\bot}
\newcommand{\opt}{o}
\newcommand{\Prob}{P}

\title{Learning Abstract World Models for Value-preserving Planning with Options}

\author{%
  Rafael Rodriguez-Sanchez\\
  Department of Computer Science\\
  Brown University\\
  Providence, RI\\
  \texttt{rrs@brown.edu} \\
  \And
  George Konidaris\\
  Department of Computer Science \\
  Brown University \\
  Providence, RI \\
  \texttt{gdk@cs.brown.edu} \\
}

\begin{document}

\maketitle

\begin{abstract}
General-purpose agents require fine-grained controls and rich sensory inputs to perform a wide range of tasks. However, this complexity often leads to intractable decision-making. Traditionally, agents are provided with task-specific action  and observation spaces to mitigate this challenge, but this reduces autonomy. 
Instead, agents must be capable of building state-action spaces at the correct abstraction level from their sensorimotor experiences. We leverage the structure of a given set of temporally-extended actions to learn abstract Markov decision processes (MDPs) that operate at a higher level of temporal and state granularity. We characterize state abstractions necessary to ensure that planning with these skills, by simulating trajectories in the abstract MDP, results in policies with bounded value loss in the original MDP.
We evaluate our approach in goal-based navigation environments that require continuous abstract states to plan successfully and show that abstract model learning improves the sample efficiency of planning and learning.
\end{abstract}

\section{Introduction}

Reinforcement learning (RL) is a promising framework for embodied intelligence because of its flexibility, generality, and online nature. Recently, RL agents have learned to control complex control systems: stratospheric balloons \citep{bellemare2020autonomous},  nuclear fusion reactors \citep{degrave2022magnetic} and drones \citep{kaufmann2023champion}. They have also mastered long-horizon decision-making problems such as the game of Go and chess \citep{silver2016mastering, silver2018general}. To achieve these results, each agent's state representation and action spaces were engineered to make learning tractable: the state space was designed to contain only relevant information for decision-making and the actions were restricted to task-relevant decisions to be made at every time step.  This is in conflict with the state-action space required for versatile, general-purpose agents (e.g., robots), which must possess broad sensory data and precise control capabilities to handle a wide variety of tasks, such as playing chess, folding clothes or navigating a maze.
Abstractions alleviate this tension: action abstractions enable agents to plan at larger temporal scales and state abstractions reduce the complexity of learning and planning; a combination of action and state abstraction results in a new task model that can capture the natural complexity of the task, instead of the complexity of the agent~\citep{Konidaris19}.

For instance, in model-based RL (MBRL; \cite{sutton1991dyna,deisenroth2011pilco}), there is a long line of research that focuses on learning transition and reward models to plan by simulating trajectories. Many modern methods learn abstract state spaces \citep{ha2018recurrent, zhang2019solar, silver2018general,hafner2019learning, hafner2020mastering, hafner2023mastering} to handle complex observation spaces. However, they learn models for the primitive action spaces and work within the single-task setting. 
Recently, there has been interest in using MBRL for skill discovery: \citet{hafner2022deep} learn a model in an abstract state space and learn a further abstraction over it to discover goals in a Feudal RL manner \citep{Dayan92}.
\cite{Bagaria20} and \cite{Bagaria21a,Bagaria21b}, instead, assume that the abstract state space is a graph and learn skills that connect nodes in that graph, effectively building a model that is both abstract in state and in actions. These approaches are ultimately limited because they assume a discrete abstract state space. 

On the other hand, in robotics, high-level planning searches for sequences of \textit{temporally-extended} actions (motor skills) to achieve a task. However, the agents needs a model to compute plans composed of their motor skills and this is typically given to the agent. To enable the agent to learn a model compatible with its motor skills from sensor data, \cite{konidaris2018skills}  propose novel semantics to automatically learn logical predicates from the agent observation space that support task planning with PDDL (Planning Domain Definition Language; \cite{fox2003pddl2, younes2004ppddl1}). Moreover, they provide theoretical guarantees for learning predicates that support sound task planning. 
%
In a similar vein, \cite{Ugur15a, Ugur15b} and \cite{Ahmetoglu22} propose to cluster the effects of motor skills to build discrete symbols for planning.  
Similarly, \cite{Asai22} introduce a discrete VAE (Variational Auto-encoder; \cite{kingma2013auto}) approach to leverage modern deep networks for grounding PDDL predicates and action operators from complex observations. 
While these approaches consider temporally-extended actions and are promising for planning problems where the appropriate state abstractions are discrete, they are not applicable when planning with the available high-level actions requires a continuous state representation.  

Instead, we are interested in learning state abstractions that are continuous, compatible with modern deep learning methods, and that guarantee value-preserving planning with a set of given skills.  Specifically, we focus on building abstract world models in the form of Markov decision processes (MDPs) that have abstract state and action spaces and, in contrast to previous approaches, provide a principled approach to characterize the abstract state space that ensures that planning in simulation with this abstract model produces a policy with expected value equal to that we would get by planning if we had access to the real MDP. In summary, we (1) introduce the necessary and sufficient conditions for constructing an abstract Markov decision process sufficient for value-preserving planning for a given set of skills; (2) introduce an information maximization approach compatible with contemporary deep learning techniques, ensuring a bounded value loss when planning using the abstract model; and finally, (3) provide empirical evidence that these abstract models support effective planning with off-the-shelf deep RL algorithms in goal-based tasks (Mujoco Ant mazes \citep{fu2020d4rl}
and Pinball \citep{KonidarisSkillChaining09} from pixels).
\begin{figure}
    \centering
    \includegraphics[width=0.6\textwidth,trim={0 0 3mm 0},clip]{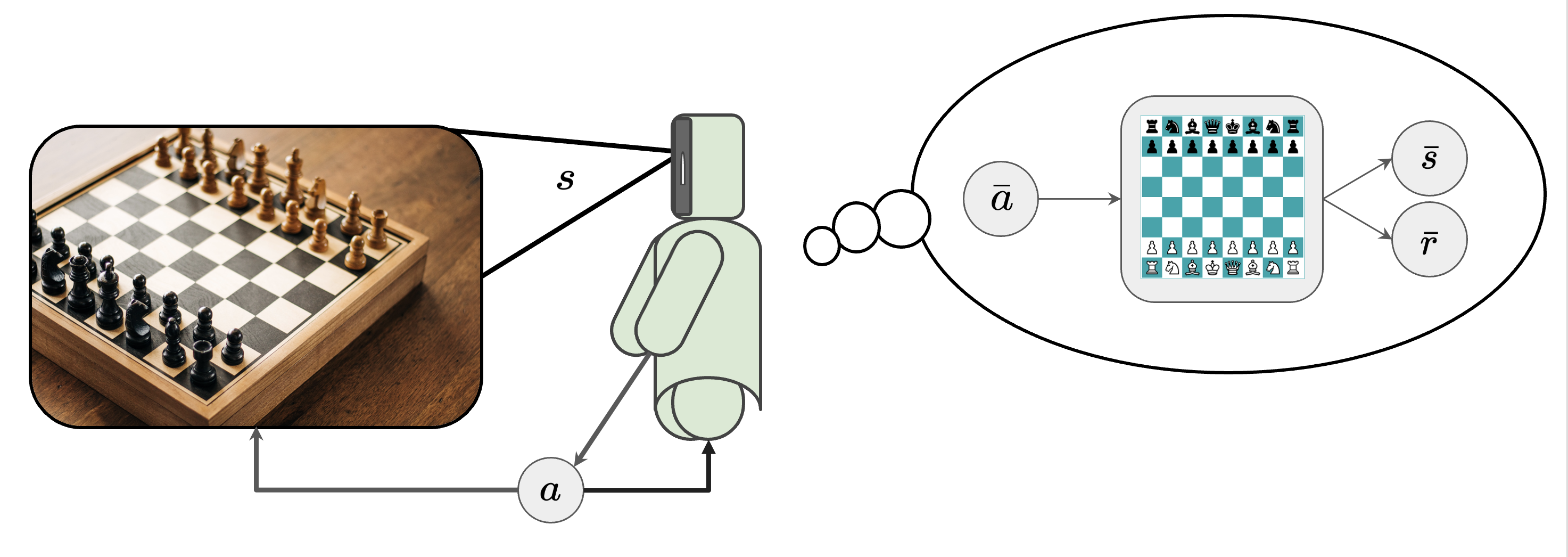}
    \caption{An agent needs to solve a task using its actuators and sensors (on the right). However, it requires an abstract model of the task (on the left) to reason at long time scales. This can be constructed by combining temporally-extended actions $\bar{a}$ with a compatible abstract state representation $\bar{s}$ that contains the minimal information necessary for planning with those actions.} 
    \label{fig:main}
\end{figure}

\section{Background and Notation}

\textbf{Markov Decision Processes}
A continuous state, continuous action Markov decision process (MDP) \citep{puterman2014markov} is defined as the tuple $ M = (\St, \A, T, R, p_0, \gamma)$  where $\St \subseteq \R^{d_s}$ is the state space and $\A \subseteq \R^{d_a}$ is the action space ($d_s, d_a \in \mathbb{N}$), $T: \St\times \A \rightarrow \Delta(\St)$\footnote{$\Delta(\cdot)$ indicates the set of probability measures over a given set.} denotes is the transition kernel that represent the dynamics of the environment, $R: \St\times \A \rightarrow \R$ is the reward function bounded by $R_{Max} \in \R$, $\gamma \in [0,1)$ is the discount factor and $p_0 \in \Delta(\St)$ is the initial state distribution. 

\textbf{Planning and Bellman Equation}
A solution to an MDP is a policy $\pi: \St \rightarrow \Delta(\A)$ that maximizes the expected return $J(\pi) = \E\left[\sum_{t=0}^\infty \gamma^t R(S_t, A_t)\right | S_0 \sim p_0, \pi, T]$. 
%
An important family of solution methods for MDPs are based on the Bellman optimality principle and the Bellman equation. 
For a given policy $\pi$, the state-value function $v^\pi : \St \rightarrow \R$ is defined as $v^{\pi}(s) := \E\left[ \sum_{t=0}^\infty \gamma^t R(S_t,A_t) | S_0=s, \pi\right]$. The state-value function represents the expected discounted return when following the policy $\pi$ from state $s$. Importantly, the value function satisfy the following recursion, known as the Bellman equation, which is used in many current planning and learning methods for MDPs: $v^{\pi}(s) = \E\left[ R(s, a) + \gamma \int_{\St} T(s'|s,a) v^{\pi}(s') ds'\right].$

\textbf{Action Abstractions} Options \citep{sutton1999between} are a formalization of temporally-extended actions, or \textit{skills}, that are used by the agent to plan with a longer temporal scope than that allowed by primitive actions. An option $\opt$ is defined by the tuple $(I_{\opt}, \pi_{\opt}, \beta_{\opt})$ where $I_{\opt}: \St\rightarrow \{0,1\}$ is the initiation set, that is, the set of states in which the option can start execution; $\pi_\opt$ is the policy function, and $\beta_\opt: \St \rightarrow [0,1]$ is the termination probability function that indicates the probability of terminating the option execution at state $s$.

\textbf{Expected-length Model of Options} Generally, options are used to plan in Semi-Markov decision processes (SMDP; \cite{sutton1999between}), in which modelling jointly the option's dynamics $T$ and duration $\tau$ as $T_{\gamma}(s'|s, \opt) = \sum_{\tau=0}^\infty \gamma^\tau \text{Pr}(S_\tau=s', \beta(s_\tau)|S_0=s, \opt)$ and its reward as $R(s, \opt) = \E_{\tau}\left[\sum_{t=0}^{\tau-1} \gamma^t R(S_t, A_t) | s, \opt\right]$, result in the Multi-time model of options. However, we will use a simpler and more practical model of option's dynamics: the expected-length model of options \citep{abel2019expected}. In this case, the option's duration is modeled independently from the next-state distribution. More precisely, let $\tilde{\tau}_{\opt}$ be the average number of timesteps taken to execute the option $\opt$, then $T_\gamma(s'|s, \opt) = \gamma^{\tilde{\tau}_{\opt}} p(s'|s,\opt)$ where $p(s'|s, \opt)$ is the probability density function over the next-state observed when the option is executed as a black-box skill.


\textbf{State Abstractions and Probabilistic Groundings} State abstractions (or state aggregation) have commonly been defined in the form of non-injective functions $f: \St \rightarrow \bar{\St}$ where $\bar{\St}$ is an abstract state space. Recently, \cite{konidaris2018skills} propose probabilistic groundings to define a new class of state abstractions. These groundings are defined by $G: \bar{\St} \rightarrow \Delta(\St)$ and, contrary to state aggregation approaches, these can have overlapping support. That is, for a state $s$ and abstract states $\bar{s}^1$ and $\bar{s}^2$, we can have that $G_{\bar{s}^1}(s) > 0$ and $G_{\bar{s}^2}(s) > 0$. In state aggregation methods, one state has just one abstract state to map to. Therefore, this provides a more expressive framework to build abstractions.

\section{Value-preserving Abstract MDPs}
To plan with a set of options, we must build a model of their effects. In this section, we formalize this model as an MDP with the following characteristics:
(1) \textbf{Action Abstraction}, the action space is the set of task-relevant temporally-extended skills (i.e., the ground actions are not used for planning);
(2) \textbf{State Abstraction}, because the set of skills operate at a higher-level of abstraction, the observation space will contains more information than required to plan with the skills;
(3) \textbf{Sufficient for Planning}, the model must support computing a plan with the option set for task-specific rewards. In the case of abstract MDPs, the abstract model must guarantee \textit{accurate} trajectory simulations to leverage the planning and RL algorithms developed for MDPs.


\subsection{Ground and Abstract MDPs}

We start by defining the ground MDP $M$, the environment that the agent observes by only executing the options. 

\begin{definition}[Ground MDP]
Let $\mathcal{O}$ be a set of options defined over the agent's state-action space. The ground MDP is $M = (\St, \mathcal{O}, T, R, \gamma, \tau, p_0)$. $T(s'|s, \opt)$ is the next-state probability density function seen by the agent when executing option $\opt$ at $s$ and its accumulated discounted reward is $R(s, \opt) = \E_{\tau}\left[\sum_{t=1}^\tau \gamma^t R(S_t, A_t) | s, o\right]$, and $\tau:\St\times\mathcal{O}\rightarrow [0,\infty)$ is the expected option's execution time of option $\opt$ when initiated at state $s$.\footnote{The ground MDP would be an SMDP if we used the multi-time model of options \citep{sutton1999between}.}
\end{definition}

\begin{definition}[Abstract MDP]
The abstract MDP is $\bar{M} = (\bar{\St}, \mathcal{O}, \bar{T}, \bar{R}, \gamma, \bar{\tau},\bar{p_0})$ where $\bar{\St}$ is the abstract state space, $\bar{T}: \bar{\St}\times O\rightarrow \Delta(\bar{S})$ is the abstract transition kernel, $\bar{R}: \bar{\St}\times \mathcal{O} \rightarrow \R$ is the abstract reward function, $\gamma$ is the discount factor, $\bar{\tau}:\bar{\St}\times \mathcal{O}\rightarrow [0, \infty)$ is the option's duration model and $\bar{p}_0$ is the initial abstract state distribution.
\end{definition}

Given that the objective is to compute plans in the abstract model, we will only consider policies of the form $\pi: \bar{\St}\rightarrow\mathcal{O}$ in the rest of the paper. Moreover, to connect the abstract MDP to the ground MDP, we use a grounding function defined in terms of probability density functions, as introduced by \cite{konidaris2018skills}. The grounding of an abstract state $\bar{s}$ is defined by the probability of the agent being in a state $s$.

\begin{definition}[Grounding function]
Let $M$ be a ground MDP and $\bar{M}$ be an abstract MDP. A grounding function $G: \bar{\St} \rightarrow \Delta(\St)$ maps $\bar{s}$ to probability measures over $\St$ of $M$. 
Given an abstract state $\bar{s}$, we denote by $G_{\bar{s}}$ its grounding probability density.
We will denote the tuple $(M, \bar{M}, G)$ as a grounded abstract model.
\end{definition}

\subsection{The Dynamics Preserving Abstraction}
Our goal is to build an abstract model that enables the agent to simulate trajectories as though it had access to a simulator of the ground model. To achieve this, we establish two key distributions: the future state distribution and the grounded future state distribution.
\begin{definition}[Future State Distribution]
    Let the tuple $(M, \bar{M}, G)$ be a grounded abstract model. Let the future state distribution be $B_t$, and defined recursively as follows,
    \begin{align*}
        &B_0(s_0) = p_0(s_0); \\
        &B_t(s_t, ..., s_0|\opt_0, ..., \opt_{t-1}) = T(s_t|s_{t-1}, \opt_{t-1})B_{t-1}(s_{t-1}, ..., s_0 | \opt_0, ..., \opt_{t-2});
    \end{align*}
    and the grounded future state distribution $\bar{B}_t$ is the estimate obtained by grounding the estimate obtained by simulating trajectories in the abstract model $\bar{M}$
    \begin{align*}
        P(s_t, \bar{s}_t, ..., s_0, \bar{s}_0|\opt_0, ..., \opt_{t-1}) &= G_{\bar{s}_t}(s_t)\bar{T}(\bar{s}_t|\bar{s}_{t-1}, \opt_{t-1})P_{t-1}(s_{t-1}, \bar{s}_{t-1}, ..., s_0,\bar{s}_0 | \opt_0, ..., \opt_{t-2}); \\
        \bar{B}_t(s_t, ..., s_0|\opt_0, ..., \opt_{t-1}) &= \int P(s_t, \bar{s}_t, ..., s_0, \bar{s}_0|\opt_0, ..., \opt_{t-1})d\bar{s}_0...\bar{s}_t;
    \end{align*}
\end{definition}
\vspace{-0.3cm}
Hence, we say that when $B_t(s_t,\dots,s_0|\opt_0,..., \opt_{t-1}) = \bar{B}_t(s_t,\dots,s_0|\opt_0,..., \opt_{t-1})$, then simulating a trajectory in the abstract model is the same as in the ground model. To satisfy this, we can build an abstract model based on dynamics-preserving abstractions.\footnote{We defer all proofs to Appendix \ref{appendix: proofs}.}
\begin{definition}[Dynamics Preserving Abstraction]\label{def:dynamics}
    Let $\phi$ be a mapping $\phi:\St\rightarrow \mathcal{Z} \subseteq \R^{d_z}$ for some dimension $d_z\in \mathbb{N}$, typically with $d_z \ll d_s$. If for all $\opt \in \mathcal{O}$ and all $s\in \St$ that are reachable with probability greater than $0$, the following holds,
    \begin{align}
        T(s'|s, \opt) &= T(s'|\phi(s), \opt);\\
        \text{Pr}(I_\opt = 1|s) &= \text{Pr}(I_\opt=1|\phi(s));
    \end{align}
    where, $I_\opt$ is an indicator variable corresponding to the option's initiation set. Then, we say that $\phi$ is dynamics-preserving. That is, the information in $\phi(s)$ is sufficient to predict the option's effect and determine if an option is executable.
\end{definition}
This is similar to model-preserving abstractions \citep{li2006towards} and bisimulation \citep{givan2003equivalence, ferns2004metrics}. However, 1) it is stronger in the sense that $z$ must be a sufficient statistic for next-state prediction, and more importantly, 2) this does not impose a condition over the ground reward function. Because we want to build an abstract model to be re-used for task-specific rewards (as we will see in Section \ref{sec:goal-based}), the ground reward function is considered as a way to measure the cost (negative reward) of executing a skill---retaining Markovianity with respect to the ground reward function would limit how much information can potentially be abstracted away.

We will now build a sensible abstract MDP $\bar{M}$, as follows. Let $\phi:\St\rightarrow \mathcal{Z}$ be a dynamics-preserving abstraction. Given that $T(s'|s, \opt) = T(s'|z, \opt)$, where $\phi(s) = z$, then we can build a transition function in $\mathcal{Z}$-space, $T(z'|z,\opt)$, and a grounding function $G$, that can let us reconstruct $T(s'|z,\opt)$.
\begin{align*}
    p_0(z) &= \int p_0(s)\mathds{1}[\phi(s)=z]ds;\\
    T(z'|z,\opt) &= \int T(s'|z, \opt)\mathds{1}[\phi(s')=z'] ds'; \\
    G(s'|z,\opt, z') &= \begin{cases}
        \frac{p_0(s') \mathds{1}[\phi(s')=z']}{p_{0}(z')} & \text{if } z' \text{is an initial state (there is not previous } (z, o)\text{)}\\
        \frac{T(s'|z, \opt) \mathds{1}[\phi(s')=z']}{T(z'|z,\opt)} & \text{otherwise}
    \end{cases};
\end{align*}
Given that just knowing $z$ is not enough to determine its grounding distribution, we can build an abstract state space $\bar{\St} \triangleq \mathcal{Z}\times \mathcal{O} \times \mathcal{Z}$ of transition tuples---with special values $z_{\False}$ and $\opt_{\False}$ to form $\bar{s}_0 = (z_{\False}, \opt_{\False}, z_0)$ for initial abstract states. Let $\bar{s} = (\hat{z}, \hat{o}, \hat{z}')$ and $\bar{s}' = (z, o, z')$ be two abstract states in $\bar{\St}$, we define the abstract MDP functions in this new $\bar{\St}$, as follows.
\begin{align*}
    G_{\bar{s}}(\cdot) &= G(\cdot|\hat{z}, \hat{\opt}, \hat{z}');\\
    \bar{T}(\bar{s}'|\bar{s}, \opt) &= 
    \begin{cases}
        T(z'|z,\opt) & \text{if } \hat{z}' = z\\
        0 & \text{otherwise}
    \end{cases}; \\
        \bar{R}(\bar{s}, \opt) &= \E_{s \sim G_{\bar{s}}}\left[R(s, \opt) \right]; \text{\hspace{0.5cm}}
    \bar{\tau}(\bar{s},\opt) = \E_{s\sim G_{\bar{s}}}\left[\tau(s, \opt)\right];\\
\end{align*}
That is, if the tuples corresponding to $\bar{s}$ and $\bar{s}'$ are not compatible, we define its transition probability as $0$, and we define the abstract reward and abstract option's execution length as their corresponding expected values under the grounding function. Finally, the following theorem formally states that this construction is sound.
\begin{theorem}
    Let the tuple $(M, \bar{M}, G)$ be a grounded abstract model and a function $\phi: \St\rightarrow \mathcal{Z}\subseteq\R^{d_z}$. The model satisfies that $B_t(\cdot \mid \opt_0,..., \opt_{t-1}) = \bar{B}_t(\cdot \mid \opt_0,..., \opt_{t-1})$ if and only if $\phi$ is dynamics-preserving.
\end{theorem}
This theorem states that if we learn a dynamics-preserving abstraction, we can simulate accurate trajectories in the abstract model. Therefore, planning in the abstract model is accurate, in the sense, that the value of an abstract state $v^\pi(\bar{s})$ computed using the abstract model is the same as the one would get by generating trajectories in the ground MDP and computing the expected value under grounding G, $\E_{s\sim G_{\bar{s}}}[v^\pi(s)]]$.

\begin{corollary}
Let the tuple $(M, \bar{M}, G)$ be a grounded abstract model. If the dynamics preserving property holds then the value of policy $\pi$ computed in abstract model $\bar{M}$ satisfies that $v^\pi(\bar{s}) = \E[v^\pi(s)|s\sim G_{\bar{s}}]$. That is, the grounded abstract model preserves the expected value under the grounding G.
\end{corollary}

\begin{proof}
Given that we have that, by definition, $T(s'|s, \opt) = T(s'|\bar{s}, \opt) = \E_{\bar{s}' \sim \bar{T}(\cdot|\bar{s},\opt)}[G_{\bar{s}'}(s)]$. It follows that
\begin{align*}
    \E_{s\sim G_{\bar{s}}}[v^\pi(s)] &= \E_{s\sim G_{\bar{s}}}\left[ \E_{\opt \sim \pi} \left[ R(s, \opt) + \E_{s' \sim T(s'|s, \opt)}\left[\gamma^{\tau}v^{\pi}(s')\right]\right]\right] \\
    &= \E_{\opt \sim \pi}\left[ \E_{s\sim G_{\bar{s}}}\left[R(s, \opt)\right] + \E_{s\sim G_{\bar{s}}, s' \sim T(s'|s, \opt)}\left[\gamma^{\tau}v^{\pi}(s')\right]\right]\\
    &= \E_{\opt \sim \pi}\left[\bar{R}(\bar{s}, \opt) + \E_{\bar{s}' \sim \bar{T}(\cdot| \bar{s}, \opt)}\E_{s'\sim G_{\bar{s}}}\left[\bar{\gamma}v^{\pi}(s')\right]\right]\\
    &= \E_{\opt \sim \pi}\left[\bar{R}(\bar{s}, \opt) + \E_{\bar{s}' \sim \bar{T}(\cdot| \bar{s}, \opt)}\left[\bar{\gamma}v^{\pi}(\bar{s}')\right]\right] = v^\pi(\bar{s}).
\end{align*}    
\end{proof}
\vspace{-0.5cm}%
The Skills to Symbols framework \citep{konidaris2018skills} introduces the strong subgoal property to build grounded discrete symbols for sound classical planning. The next corollary proves that the strong subgoal is a special case of the dynamics preserving property when the appropriate abstraction function has finite co-domain. Therefore, we can build discrete dynamics preserving models if and only if the strong subgoal property holds.

\begin{corollary}
    Let the tuple $(M, \bar{M}, G)$ be a grounded abstract model. 
    Let the strong subgoal property \citep{konidaris2018skills} for an option $\opt$ be defined as, $\text{Pr}(s'|s, \opt) = \text{Pr}(s'|\opt)$.
    The dynamics preserving property holds with a finite abstract state space $\mathcal{Z} = [N]$ for some $N \in \mathbb{N}$  if and only if the strong subgoal property holds.
\end{corollary}

\section{Learning the Abstract Model}

\subsection{Information Maximization to Learn a Dynamics-Preserving $\phi$}

The mutual information (MI) between random variables $X$ and $Y$, $MI(X;Y)$,
measures the information that each variable holds about the other. 
We are interested in finding a function $\phi$ that is dynamics-preserving such that we can build our abstract MDP. By Definition \ref{def:dynamics}, we want to learn $\phi(s)$ that is maximally predictive of the effect of $\opt$ when executed in $s$ and to predict if option $\opt$ is executable. That is, we want to maximize the following:
\begin{align}
    \max_{\phi\in\Phi} MI(S', I; \phi(S), O) \equiv\max_{\phi\in\Phi} MI(S'; \phi(S), O) + MI(I; \phi(S)), \label{eq:objective}
\end{align}
where $\Phi$ is a class of functions that map the high-dimensional ground states to lower-dimensional space. $I$ is binary random variable for the initiation set prediction.  $S', S, O$ are random variables over the ground states $\St$ and the options set $\mathcal{O}$. 

In general, by the data processing inequality, $MI(S'; \phi(S), O)$ is upper-bounded by $MI(S'; S, O)$. 
Therefore, we can show that optimizing the above objective results in a bounded value loss when using the abstract model to plan. To see this, we first note that by compressing through $\phi$, we lose information $\Delta MI \triangleq MI(S';S, O) - MI(S'; Z, O)$, where $Z=\phi(S)$, in the transition dynamics simulation. 
We show that,
\begin{align*}
    \Delta MI  
    \overset{(a)}{=} \E_{p(s)} \left[ D_{KL}\left(T(s'|s, \opt)|| \tilde{T}(s'|z, \opt)\right)\right] 
     \overset{(b)}{\geq}2\ln 2\cdot\E_{p(s)}\left[ \lVert T(s'|s, \opt) - \tilde{T}(s'| z, \opt)\rVert_1^2\right].
\end{align*}

where $p(s)$ is a distribution over $s$ that will depend on the data collection policy and (a) follows from the definition of the KL divergence and (b) from the well-known bound relating the KL divergence and L1 norm\footnote{$D_{KL}(P, Q) \geq 2\ln2 \cdot\lVert P-Q\rVert_1^2$}. Therefore, the error in the learned transition dynamics is minimized by our objective and this implies, by the following theorem,  that this objective also minimizes the value loss resulting from the approximation.

\begin{theorem}[Value Loss Bound]\label{thm:value-loss}
    Let $(M, \bar{M}, G)$ be a grounded abstract model and $\tilde{T}(s'|\bar{s}, \opt) = \int G_{\bar{s}'}(s') \bar{T}(\bar{s}'|\bar{s}, \opt)d\bar{s}'$ be the approximate transition dynamics from the grounded model. If the following conditions hold for all $\opt \in \mathcal{O}$ and all $s\in \St$ with $G_{\bar{s}}(s) > 0$: (1)
        $\lVert T(s'|s, \opt) - \tilde{T}(s'|\bar{s},\opt)\rVert_1^2 \leq \epsilon_T$, and 
        (2)$\lvert R(s, \opt) - \bar{R}(\bar{s}, \opt)\rvert^2 \leq \epsilon_R$;
    then, for any policy $\pi$,
    \begin{equation*}
        \lvert Q^\pi(s, \opt) - Q^\pi(\bar{s}, \opt)\rvert \leq \frac{\sqrt{\epsilon_R} + \gamma V_{Max}\sqrt{\epsilon_T}}{1-\gamma}.
    \end{equation*}
\end{theorem}

\subsection{Contrastive Abstract Model Learning}

We maximize the previous Infomax objective (\ref{eq:objective}) as follows. The term $MI(I;Z)$ reduces to a cross entropy loss, so we will focus on estimating the term $MI(S';Z,O)$: 
%
we can prove that maximizing both sides of the identity $MI(Z';Z,O) = (MI(S';Z') - MI(S';Z'|Z,O))$ implicitly maximizes $MI(S'; Z, O)$ (see extended derivation details in Appendix \ref{appendix:tpc-derivation}). Intuitively, the first term $MI(Z';Z,O)$ makes $z'$ predictable from knowing the option executed and the previous $z$. The second term avoids collapsing $\phi$ to a trivial solution: maximizing $MI(S';Z') - MI(S';Z'|Z,O)$ makes $\phi$ retain information about the ground state $s$ (avoiding collapse of the representation) that is maximally predicted by the previous $(z,o)$. 

\begin{wrapfigure}{r}{0.5\linewidth}
  \begin{minipage}{.5\textwidth}
    \begin{algorithm}[H]
      \caption{Planning and Learning with an Abstract Model}
      \label{alg:planning}
      \begin{algorithmic}[1]
        \REQUIRE Agent $\pi$, Ground Environment M, \\ Abstract Model $\bar{M}$, Goal $\mathcal{G}$
        \STATE Initialize dataset $\mathcal{D}$ by rolling out $N$ trajectories
        \STATE $\bar{M}\leftarrow$ PretrainAbstractMDP($\mathcal{D}$)
        \STATE $\bar{M}\leftarrow$ MakeTaskMDP($\bar{M}$, $\mathcal{G}$)
        \WHILE{true}
          \STATE $\mathcal{D}$ $\leftarrow$ Roll out for $L$ steps. 
          \IF{$H$ steps have passed}
            \STATE $\bar{M}\leftarrow$ TrainModel($\bar{M}$,$\mathcal{D})$
            \STATE $\pi\leftarrow$ TrainAgentImagination($\bar{M}$, $\pi$)
          \ENDIF
        \ENDWHILE
      \end{algorithmic}
    \end{algorithm}
  \end{minipage}
  \vspace{-0.5cm}
\end{wrapfigure}
We choose to maximize these mutual information terms contrastively using InfoNCE \citep{oord2018cpc} to avoid making assumptions about tractable density models (other MI estimators \citep{poole2019variational, alemideep, belghazi2018mutual} can be used). Using these estimators allows the model to implicitly learn complex grounding functions that improve the quality of the abstract state space.
Note that by using InfoNCE for the terms above, this algorithm corresponds to Temporal Predictive Coding (TPC; \cite{nguyen2021temporal}) which proposes abstract states without reconstruction objectives. Therefore, our formulation corresponds to the TPC algorithm in the degenerate case of options being the primitive actions.\footnote{Extended discussion in Appendix \ref{appendix:tpc-derivation}}

In practice, we assume that we have access to a dataset of transition samples $\mathcal{D} = \{(s_i, \opt_i, r^\gamma_i, s'_i, \tau_i, I_i)\}_{i=1}^N$ that correspond to the execution of option $\opt_i$ from state $s_i$, terminating in $s'_i$ with a duration of $\tau_i$ and accumulated return $r_i^\gamma = \sum_{t=0}^{\tau_i-1} \gamma^t r_t$. $I_i$ corresponds to the initiation sets of all options in state $s_i$. This dataset might be initialized by rolling out trajectories with a random agent and further enhanced during the agent's learning (see Algorithm \ref{alg:planning}).

We propose to learn the abstract model $M^\phi_\theta = (T^\phi_\theta, R_\theta, I^\phi_\theta, \tau_\theta)$ based on the abstraction $\phi$ parameterized by a function approximator $f_\phi$.  Notice, that because we need to guarantee good initiation sets by $MI(I; \phi(S))$, the initiation set loss also affects the learning of $f_\phi$:
\begin{align*}
    \mathcal{L}_\phi &= -MI_\phi(Z'; Z, O) - MI_\phi(S';Z'); \\
    \mathcal{L}_{\theta,\phi}^I &=  -\log I_\theta^\phi(I_i|f_\phi(s_i)); \\
    \mathcal{L}_{\theta, \phi}^T &= -\log\bar{T}_\theta(f_\phi(s'_i)|f_\phi(s_i), \opt_i);
\end{align*}
Therefore, $\mathcal{L}_\phi$, $\mathcal{L}_{\theta,\phi}^I$ and $\mathcal{L}_{\theta, \phi}^T$ are used to learn the abstraction function $f_\phi$. Moreover, to compensate for any imbalances in the data, we use a weighted negative log-likelihood loss for the initiation loss to learn an initiation classifier to be used during planning. 
To learn the rest of the model, we consider $f_\phi$ fixed and minimize the following losses and consider samples of the form $(s_{i-1}, \opt_{i-1}, s_{i}, \opt_{i}, r_i^\gamma, \tau_i)$ which can be obtained by slicing trajectories appropriately. We map them considering $f_\phi$ and minimize the following,
\begin{align*}
    \mathcal{L}_\theta^R &= (R_\theta(z_{i-1}, \opt_{i-1}, z_i, \opt_{i})-r^\gamma_i)^2; \text{\hspace{0.5cm}}
    \mathcal{L}_\theta^\tau = (\tau_\theta((z_{i-1}, \opt_{i-1}, z_i, \opt_{i})-\tau_i)^2;
\end{align*}
Finally, we minimize $\mathcal{L}_{\theta, \phi} = \beta_{\text{info}} \mathcal{L}_\phi + \beta_{I}\mathcal{L}_{\theta,\phi}^I + \beta_{T} \mathcal{L}_\theta^T + \beta_{R}\mathcal{L}_\theta^R + \beta_{\tau} \mathcal{L}_\theta^\tau$. In our experiments, all constants were $\beta_{\text{info}} = \beta_I = \beta_T = \beta_R = \beta_\tau = 1$. 
\subsection{Goal-based Planning with an Abstract Model}\label{sec:goal-based}
Consider a goal set $\mathcal{G}\subset S$ and $\mathcal{G}_\phi \subset \mathcal{Z}$, its mapping to $\mathcal{Z}$. In order to define the task MDP  $M_{\mathcal{G}}$ (Algorithm \ref{alg:planning}, Line 3) for the agent to plan in,
we define the task reward function for abstract state $\bar{s} = (\hat{z}, \hat{\opt}, z)$ as $R_{\mathcal{G}}(\bar{s}, \opt) = R_\theta(\bar{s}, \opt) + R_{\text{task}} \mathds{1}[z \in \mathcal{G}_\phi] $ where $R_{\text{task}}$ is the goal reward. The first term can be interpreted as the base cost/reward of executing a skill while the second term indicates to the agent the task-specific rewarding states. Moreover, we augment the transition dynamics and set all $z \in \mathcal{G}_\phi$ as terminating states by setting $\bar{T}_{\mathcal{G}}(z_{\text{done}}|z, \opt) = \mathds{1}[z\in \mathcal{G}_{\phi}]$.  
The agent uses the task MDP $\bar{M}_{\mathcal{G}}$ to simulate trajectories and improves its policy (Algorithm \ref{alg:planning}, Line 8) and it can rollout the policy in the environment to collect new data (Algorithm \ref{alg:planning}, Line 7) that further improves the abstract model.

\section{Experiments}

\textbf{Pinball environment} \citep{KonidarisSkillChaining09} This domain has a continuous state space  with position vector $(x,y) \in [0,1]^2$ and velocities $(\dot{x}, \dot{y}) \in [-1, 1]^2$. As opposed to its original formulation, we consider a variant with continuous actions that decrease or increase the velocity by $\Delta(\dot{x}, \dot{y}) \in [-1, 1]^2$. Moreover, we also consider the top view pixel observation of the environment as the agent's observation.
As options, we handcrafted position controllers implemented as PID controllers that move the ball in the coordinate directions by a fixed step size. 
\begin{wrapfigure}{r}{0.5\textwidth}
    \centering
    \includegraphics[width=0.4\textwidth]{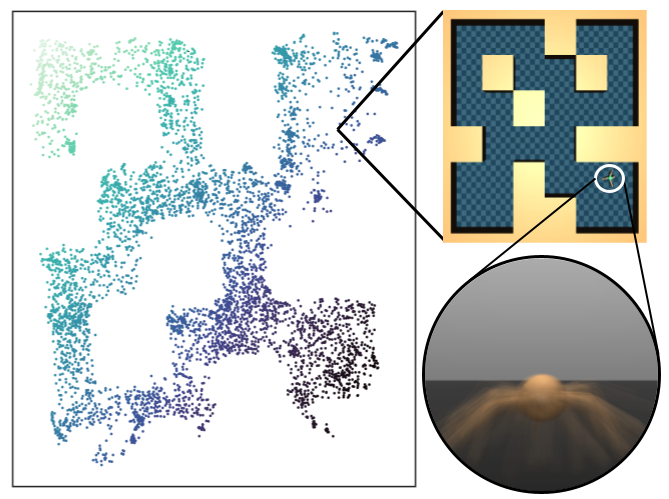}
    \caption{\textbf{Medium Antmaze}. 2D MDS projection of the learned $\phi$: it learns to represent the position in the maze. The average grounding shows possible configurations of the ant joints when it is in the represented position.\vspace{-0.5cm}}\label{fig:antmaze}
\end{wrapfigure}
\textbf{Antmaze} We consider the problem of controlling a Mujoco \citep{todorov2012mujoco, fu2020d4rl} Ant to navigate through a maze. The state space is a $29$-dimensional vector that contains the position of the ant in the maze and the ant's proprioception. We consider the Medium Play maze as defined by \cite{fu2020d4rl}.
We use $8$ options learned using TD3 \citep{fujimoto2018addressing} that move the ant in the coordinate directions (north, south, east, west and the diagonal directions) in the maze by a fixed step size.
 
\subsection{Abstract State Space Preserves Relevant Information for Planning}

\begin{figure}
    \centering
  \begin{subfigure}[b]{0.4\textwidth}
    \centering
    \includegraphics[width=0.7\textwidth]{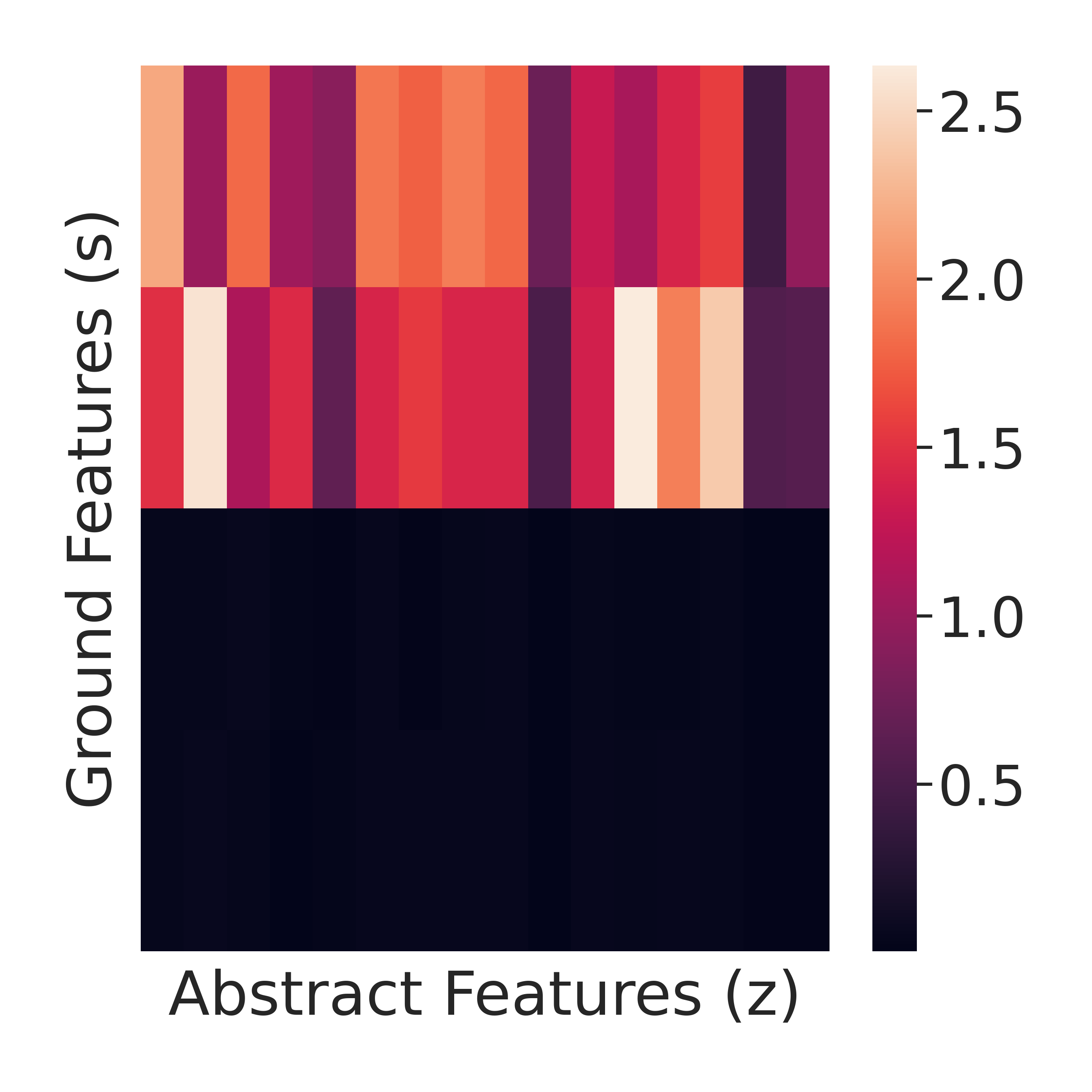}
    \caption{Pinball}\label{fig:pinball-mi}
  \end{subfigure}
  \begin{subfigure}[b]{0.4\textwidth}
    \centering
    \includegraphics[width=0.7\textwidth]{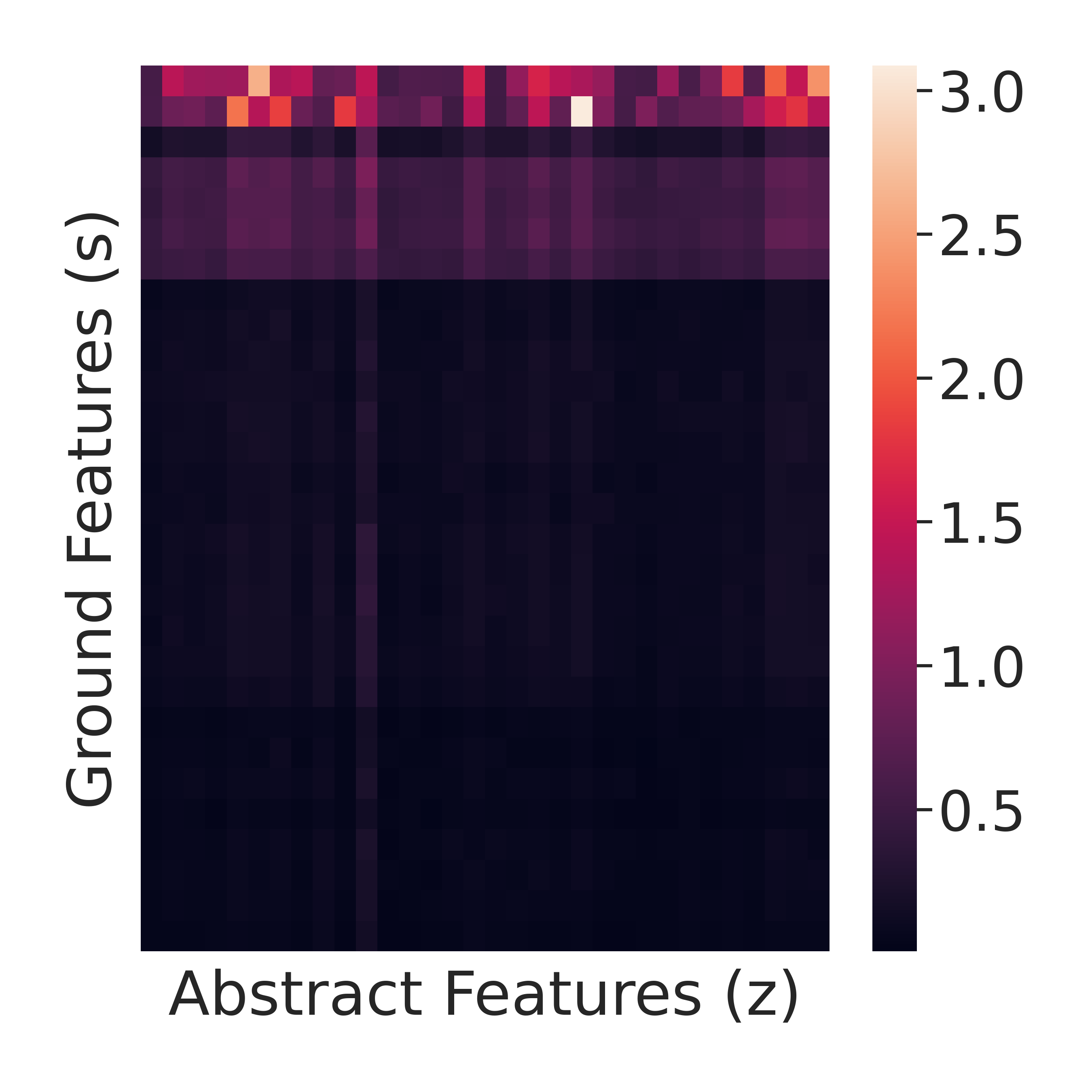}
    \caption{Medium Antmaze}\label{fig:antmaze}\label{fig:antmaze-mi}
  \end{subfigure}
\caption{MI matrix: ground features $s$ are in the vertical axis and abstract features $z$ are in the horizontal axis. High MI (first two rows) corresponds to the position of the ball or the ant.}\label{fig:mi}
\end{figure}

Our main hypothesis is that abstract actions drive state abstraction because the information needed to plan with a structured option set will be less than the ground perception space of the agent.
To quantify this, we measure the information contained in the abstract state space about the ground features by estimating the MI using non-parametric methods based on $k$-nearest neighbors \citep{kozachenko1987sample}. We use Scikit-learn implementation \citep{scikit-learn}. 
In Figure \ref{fig:pinball-mi}, we show the MI matrix between Pinball's ideal features (position and velocities) and the learned features from the pixel observations.
For Antmaze (Figure \ref{fig:antmaze-mi}), we purposely over-parameterized the abstract space to give enough capacity to learn the full observation, if necessary. However, we can see that features that are not necessary for planning with the skills are effectively abstracted away. 
In the case of Pinball only the first two dimensions corresponding to the ball position have high MI. In the Antmazes, similarly, the first $7$ dimensions have the highest MI which corresponds to position in the maze (first two dimensions) and orientation of the ant's torso.
Qualitatively, we can visualize the learned abstract state space using Multidimensional Scaling (MDS; \cite{borg2005modern}). Figure \ref{fig:antmaze} shows the abstract state space learned for the Antmaze and it reveals the pattern of the coordinate positions of the ant in the maze. Additionally, we show grounded observations that correspond to an abstract state: the ant at the represented position in the maze with many different configurations of the joints and torso.

\subsection{Planning with an Abstract MDP}

To evaluate the effectiveness of these models for multiple goal-based tasks, we pretrained abstract models and use them to plan in imagination using Double DQN \citep{van2016deep}: the DDQN agent rolls out imagined trajectories to improve its policy and then rolls it out in the ground environment to collect new data that is used to learn the task reward function (we keep fix the rest of the model).
As our baseline, we use DDQN tuned to learn a policy with the same options but interacting with the ground MDP.
In Figure \ref{fig:planning}, we show learning curves (success rate vs. ground environment steps) averaged over different goals and seeds. The error areas represent one standard deviation.

For the pinball domain we use pixel observations as input. In Figure \ref{fig:pinball-plan}, we compare learning curves averaged over $8$ goals and $5$ seeds where the gray area represent the number of samples used for pretraining phase of the model. These curves show that planning in the abstract model achieves similar performance to the same agent learning directly in the ground MDP which showcases the gain obtained in terms of sample efficiency.

Figure \ref{fig:maze-plan} shows an analogous plot for Antmaze ($9$ goals and $5$ seeds).
In this domain we provide additional results for state-of-the-art model-based RL methods: DreamerV2 and DreamerV3 \citep{hafner2020mastering, hafner2023mastering}. These methods have been shown to work in diverse domains by building (discrete) latent states based on reconstruction losses. However, their performance is limited in comparison to our abstract model: (1) notice that after the gray area our abstract model collects data only to improve the goal reward prediction, whereas the baselines continuously collect data that further improves their models which shows the sample efficiency afforded by our skill-driven abstraction, and (2) our simple DDQN agent learns faster in imagination that the more sophisticated planning agents of the baselines.

%
%
%
\begin{figure}[h]
  \centering
 \includegraphics[width=0.4\textwidth]
    {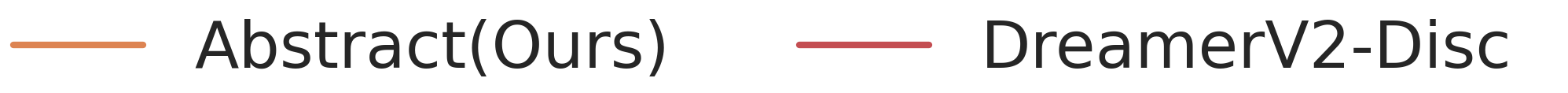}
    \includegraphics[width=0.4\textwidth]
    {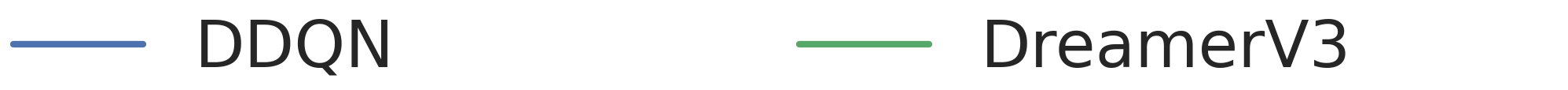}\\
  \begin{subfigure}[b]{0.49\textwidth}
    \centering
    \includegraphics[keepaspectratio=true, width=0.7\textwidth]
    {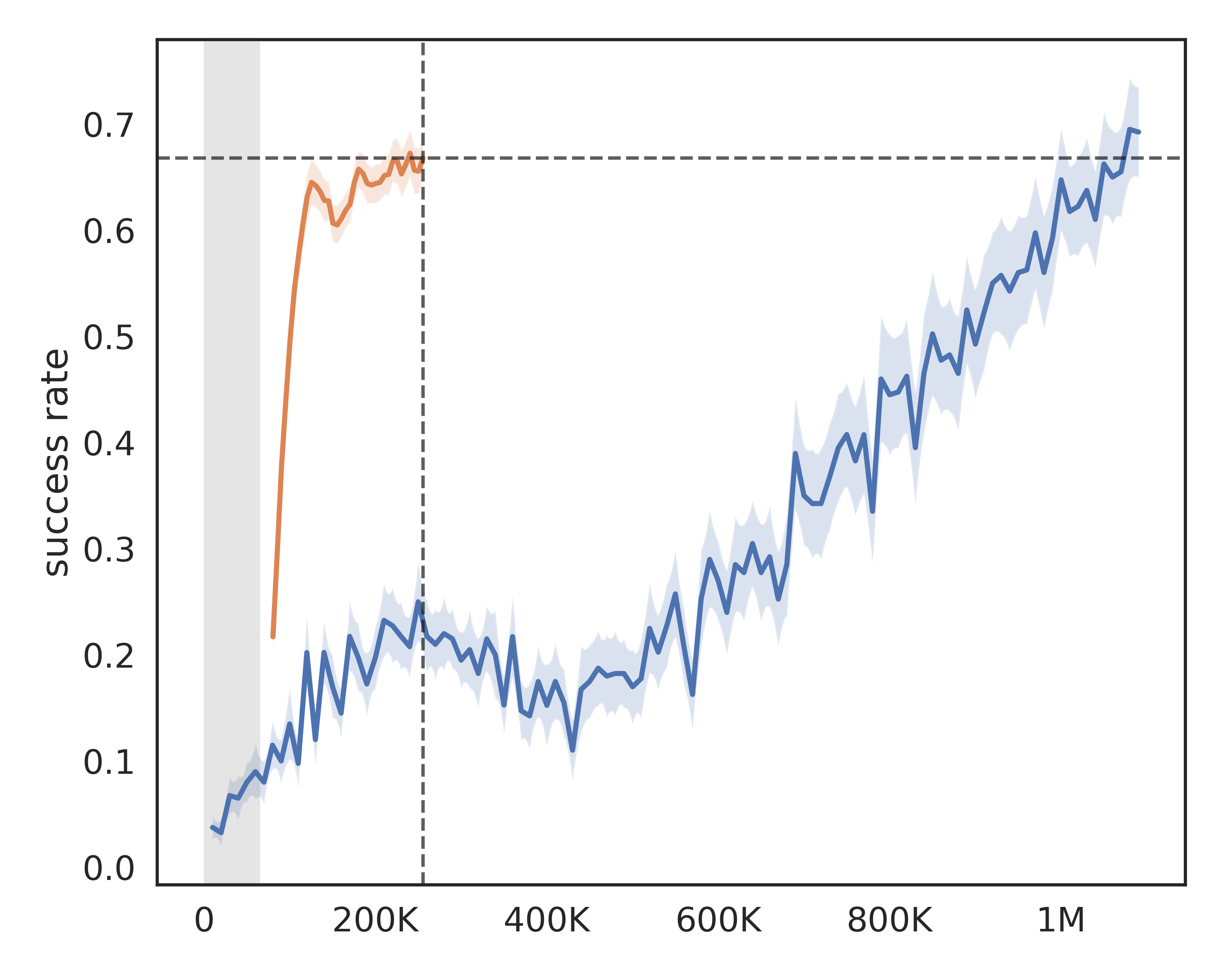}
    \caption{Pinball}\label{fig:pinball-plan}
  \end{subfigure}
  \begin{subfigure}[b]{0.49\textwidth}
    \centering
    \includegraphics[keepaspectratio=true, width=0.7\textwidth]
    {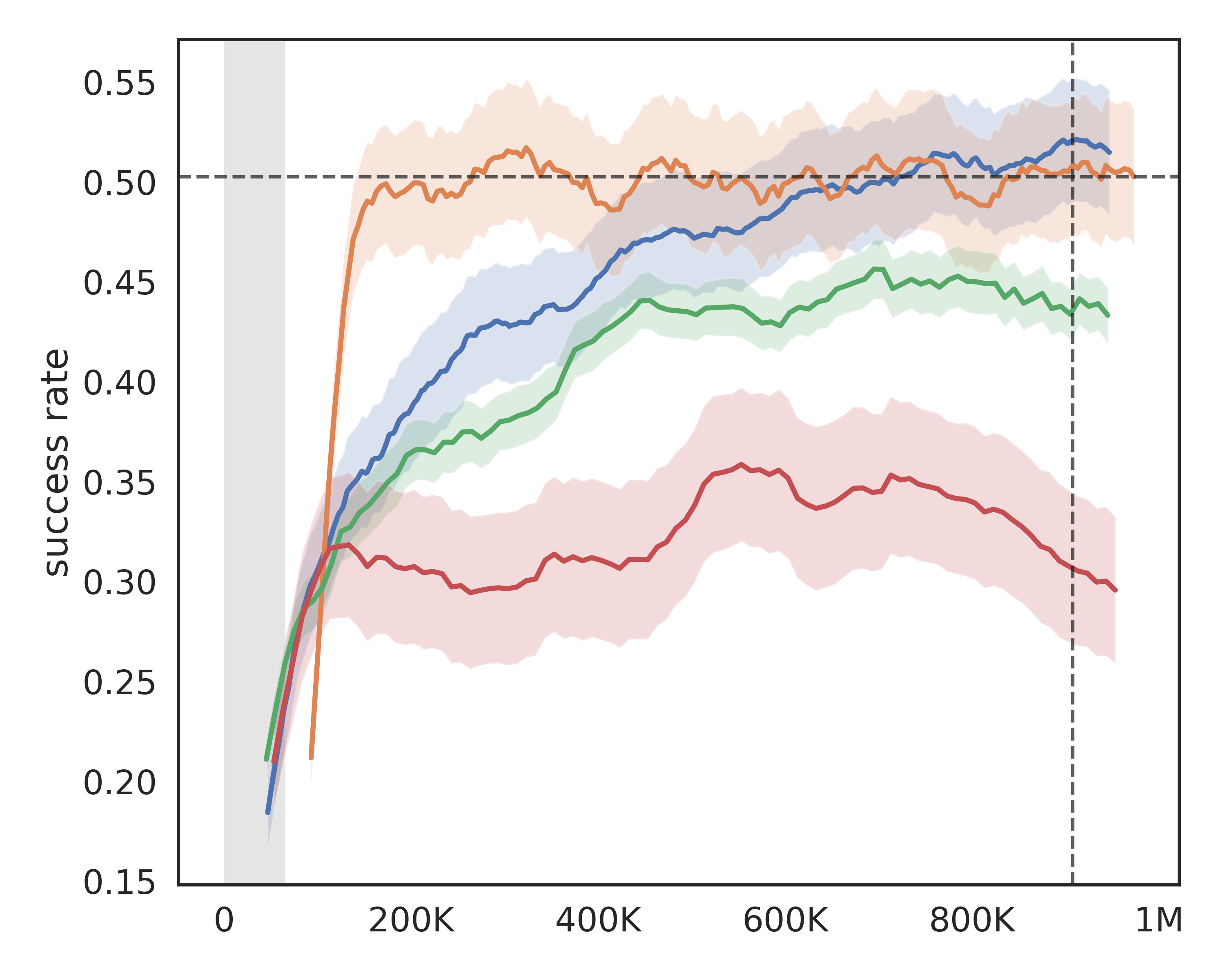}
    \caption{Antmaze}\label{fig:maze-plan}
  \end{subfigure}
  \caption{Planning with an abstract model. Success rate v. Environment steps averaged over goals and $5$ seeds. The gray area represents the offset for the steps needed to pre-train the model.\vspace{-0.5cm}}\label{fig:planning}
\end{figure}
\section{Related Works}
\vspace{-0.25cm}
\textbf{Grounded Classical Planning} 
\cite{konidaris2018skills} present a skill-driven method for constructing PDDL  predicates \citep{fox2003pddl2, younes2004ppddl1} for classical planning. This family of work formally bridges the options framework to classical planning, and recent work have extended this framework to work with portable skills \citep{James20} and object-centric skills \citep{James22}, and to ground natural language in robotics \citep{Gopalan17}. 
Importantly, this framework offers guarantees that the learned grounded symbols support sound planning.
A related body of work bridges deep learning with classical planning. \cite{Asai18, Asai19, Asai22} learn abstract binary representations to ground PDDL predicates and action operators from complex observations. Similarly, \cite{Ugur15a, Ugur15b} approach the grounding problem by clustering action effects to create discrete symbols for planning, and \cite{Ahmetoglu22} extends this approach to leverage deep learning methods. While these approaches manage to empirically work with complex observations, they do not offer formal guarantees that the symbols learned are sufficient for planning.
Our approach, while not applied to classical planning, generalizes abstract state learning to continuous cases, it is compatible with the deep learning toolbox and it is theoretically principled.

\textbf{Model-based RL and State Abstractions} Learning MDP models from experience has been extensively studied \citep{sutton1991dyna, deisenroth2011pilco} for their benefits in generalization, sample efficiency, and knowledge transfer. Recent successful approaches use deep networks to handle complex observations spaces and long-term reasoning \citep{krishnan2015deep,  ha2018recurrent, silver2018general, gregor2018temporal, buesing2018learning, zhang2019solar, hansen2022temporal, hansen2023td}. An important challenge of this approach is learning an effective abstract state space and, most of them, have focused in learning abstract representations of complex observations based on reconstruction losses \citep{gregor2018temporal, buesing2018learning, zhang2019solar, hafner2019learning, hafner2020mastering, hafner2023mastering}. 
In contrast, recent approaches have moved away from this idea and focused in minimal abstract state spaces relevant for acting such as value prediction \citep{silver2018general, grimm2020value, yue2023value}, Markov states \citep{gelada2019deepmdp, zhang2020learning, allen2021learning, nguyen2021temporal}, and controllability \citep{lamb2022guaranteed}. In fact, many of these  explicitly use information maximization and information bottleneck approaches that are theoretically justified by our work.

From a theoretical point of view, there is extensive research to characterize the types of state abstractions (or state aggregation) \citep{li2006towards, ferns2004metrics, castro2010bisim} that are useful for RL. More recent work characterizes \textit{approximate} state abstractions \citep{abel2016near, abel2018state} that guarantee bounded value loss and the type of options that are compatible with a given state abstraction to guarantee value preservation \citep{abel2020value}.

\textbf{Temporally-extended Models} MDP models with skills have been recently considered in skill discovery research. Some work approach the problem assuming that the abstract state space is a graph and options are learned to reach the initiation set of another option \citep{Bagaria20, Bagaria21a, Bagaria21b}. \cite{hafner2022deep} approaches the problem by building on the Dreamer algorithm \citep{hafner2019learning, hafner2020mastering, hafner2023mastering} and discover goals by abstracting over the learned abstract state. Similarly, \cite{nair2019hierarchical} use generative models for subgoal generation and skill learning, and plan with a learned model in observation space. Other approaches learn forward dynamics models for skills discovered from an offline set of trajectories but do not abstract the state based on these skills \citep{oposm2023, shi2023skill, zhang2023leveraging}. While our method assumes that the options are given, it does not impose discrete constraints to the abstract state space, does not need to model the state dynamics at the finest time step, and it builds a principled abstract state space.

\section{Conclusion}

We introduce a method for learning abstract world models, designed to have agents with effective planning capabilities for goal-oriented tasks. Our core premise is that an agent must be capable of building a reusable abstract model for planning with a given skill set. We do this in a principled manner by  characterizing the state abstraction that guarantees that planning in simulation guarantees bounded value loss. In other words, planning with a learned abstract model is sufficient to compute a policy for the real-world environment.

\subsubsection*{Acknowledgments}

We  would like to thank Sam Lobel, Akhil Bagaria, Saket Tiwari, and other members in the IRL at Brown for useful discussions during the development of this project. Moreover, we would like to thank David Abel for his contribution to the Value Loss theorem. 
This project was funded by NSF grant \#1955361, NSF CAREER \#1844960 to Konidaris, ONR grant \#N00014-22-1-2592. Partial funding for this work provided by The Boston Dynamics AI Institute (``The AI Institute'').

\bibliographystyle{plainnat}
\bibliography{ref.bib}


\newpage
\appendix

\section{Appendix}

\subsection{Proofs}\label{appendix: proofs}



\begin{theorem}
    Let the tuple $(M, \bar{M}, G)$ be a grounded abstract model and a function $\phi: \St\rightarrow \mathcal{Z}\subseteq\R^{d_z}$. The model satisfies that $B_t(\cdot \mid \opt_0,..., \opt_{t-1}) = \bar{B}_t(\cdot \mid \opt_0,..., \opt_{t-1})$ if and only if $\phi$ is dynamics preserving.
\end{theorem}

\begin{proof}

Let $\phi^{-1}(z) = \{s\in\St \mid \phi(s)=z \}$.  We construct $\bar{T}$ and $G$  such that it satisfies that,
\begin{align*}
    \bar{T}(z'|z, \opt) &= \int_{s'\in \phi^{-1}(z')} T(s'|z, \opt)ds';\\
    G(s'|z, \opt, z') &= \frac{T(s'|z, \opt) \mathds{1}[\phi(s')=z']}{\bar{T}(z'|z,\opt)}
\end{align*}

If the dynamics preserving property holds, we have that there exists a mapping $\phi$ such that $T(s'|s,\opt) = T(s'|\phi(s),\opt)$.
Hence, by defining that abstract state as $\bar{s} = (z, \opt, z')$, we can build the grounded abstract model such that it follows that $B_t = \bar{B_t}$, by construction.

To prove the converse, we assume that $B_t = \bar{B}_t$. 

Hence, by construction, we have that $\Prob(s_t, ..., s_0|\opt_0, z_0, ..., \opt_{t-1}, z_{t-1}) = \prod_t \Prob(s_t|\opt_0, z_0, ..., \opt_{t-1}, z_{t-1})$. Therefore, we have that

\begin{align*}
    \bar{B_t}(s_t,..., s_0|\opt_0, ..., \opt_{t-1}) &= \int \prod_{i=0}^t \Prob(s_i|\opt_0, z_0, ..., \opt_{i-1},z_{t-1})\Prob(z_i, ..., z_0|\opt_0,..., \opt_{i-1})dz_0...z_t\\
    &= \int \prod_{i=0}^t \Prob(s_i|z_i, \opt_{i-1})\Prob(z_i, ..., z_0|\opt_0,..., \opt_{i-1})dz_0...z_t\\
    &= \int \prod_{i=0}^t G(s_i|z_{i-1}, \opt_{i-1}, z_i)\Prob(z_i, ..., z_0|\opt_0,..., \opt_{i-1})dz_0...z_t\\
    &= \prod_{i=0}^t \int G(s_i|z_{i-1}, \opt_{i-1}, z_i)\Prob(z_i, z_{i-1}|\opt_0,..., \opt_{i-1})dz_iz_{i-1}\\
    &= \prod_{i=0}^t \int G(s_i|z_{i-1}, \opt_{i-1}, z_i)\bar{T}(z_i|z_{i-1}, \opt_{i-1})\Prob(z_{i-1}|\opt_0,..., \opt_{i-2})dz_iz_{i-1}\\
    &= \prod_{i=0}^t \int \tilde{T}(s_i|z_{i-1}, \opt_{i-1})\Prob(z_{i-1}|\opt_0,..., \opt_{i-2})dz_{i-1}
\end{align*}
\begin{align*}
    B_t(s_t, ..., s_0|\opt_0, ..., \opt_{t-1}) &= p_0(s_0)\prod_{i=1}^t  T(s_i|s_{i-1}, \opt_{i-1})\\
    &= \prod_{i=1}^t  T(s_i|s_{i-1}, \opt_{i-1})\Prob(s_{i-1}|\opt_0, ...,\opt_{t-2})
\end{align*}

Hence, we must have that for all $s_{i-1} \in z_{i-1}$ and all $i \in [t]$ and $t \geq 0$

\begin{equation*}
    \int T(s_i|s_{i-1}, \opt_{i-1})\Prob(s_{i-1}|\opt_0, ...,\opt_{t-2})ds_{i-1} = \int \tilde{T}(s_i|z_{i-1}, \opt_{i-1})\Prob(z_{i-1}|\opt_0,..., \opt_{i-2})dz_{i-1}
\end{equation*}

That is, 
\begin{align*}
    \begin{cases}
        \Prob(s_0) = p_0(s_0) = \int G(s|z_0)p_0(z_0)ds & \text{for  } t = 0\\
        \Prob(s_1|\opt_0) = \int T(s_1|s_0, \opt_0)p_0(s_0)ds_0 = \int \tilde{T}(s_0|z_0, \opt_0)p_0(z_0)dz_0 & \text{for  } t=1 \\
    \end{cases}
\end{align*}

By definition, $t=0$ holds. For $t=1$, we have

\begin{align*}
    \Prob(s_1|\opt_0) &= \int T(s_1|s_0, \opt_0)p_0(s_0)ds_0\\
    &= \int T(s_1|s_0, \opt_0)G(s_0|z_0)p_0(z_0)dz_0ds_0 \\
    &= \int \tilde{T}(s_1|z_0, \opt_0)p_0(z_0)dz_0
\end{align*}

which follows from the equation at $t=0$. Hence, it must be true that for any $s_0 \in \phi^{-1}(z_0)$, for any $z_0$ with $p_0(z_0) > 0$. 

\begin{align*}
\tilde{T}(s_1|z_0, \opt_0) = \int T(s_1|s_0, \opt_0)G(s_0|z_0)ds_0
\end{align*}

We can see that for any $s_0 \in \phi^{-1}(z_0)$ such that $T(s_1|s_0, \opt_0) \neq \tilde{T}(s_1|z_0, \opt_0)$, the abstract model would commit a non-zero error in its prediction. Hence, it must be that $T(s_1|s_0, \opt_0) = \tilde{T}(s_1|z_0, \opt_0)$ for $s_0 \in \phi^{-1}(z_0)$.

Let the equations at time $t=i-1$ and $t=i-2$ hold, then
\begin{align*}
    &\Prob(s_i|\opt_0, ..., \opt_{i-1}) = \int T(s_i|s_{i-1}, \opt_{i-1})p_{i-1}(s_{i-1}|\opt_0,...\opt_{i-2})ds_{i-1}\\
    &= \int T(s_i|s_{i-1}, \opt_{i-1})\tilde{T}(s_{i-1}|z_{i-2}, \opt_{i-2})p_{i-2}(z_{i-2}|\opt_0, ...,\opt_{i-3})ds_{i-1}dz_{i-1}dz_{i-2} \\
    &= \int T(s_i|s_{i-1}, \opt_{i-1})G(s_{i-1}|z_{i-2}, \opt_{i-2}, z_{i-1})\bar{T}(z_{i-1}|z_{i-2}, \opt_{i-2})p_{i-2}(z_{i-2}|\opt_0, ...,\opt_{i-3})ds_{i-1}dz_{i-1}dz_{i-2}\\
    &= \int \tilde{T}(s_i|z_{i-1}, \opt_{i-1})p_{i-1}(z_{i-1}|\opt_0, ...,\opt_{i-2})dz_{i-1}
\end{align*}

Because $p_{i-1}(z_{i-1}|\opt_0, ...,\opt_{i-2}) = \int \bar{T}(z_{i-1}|z_{i-2}, \opt_{i-2})p_{i-2}(z_{i-2}|\opt_0, ...,\opt_{i-3})dz_{i-2}$ hold by construction of the abstract MDP, we need the following to hold.
\begin{equation}
    \tilde{T}(s_i|z_{i-1}, \opt_{i-1}) = \int T(s_i|s_{i-1}, \opt_{i-1})G(s_{i-1}|z_{i-2}, \opt_{i-2}, z_{i-1})ds_{i-1}.
\end{equation}

Therefore, as in the base case, we need that $\tilde{T}(s_i|z_{i-1}, \opt_{i-1}) = T(s_i|s_{i-1}, \opt_{i-1})$ for all $s_{i-1} \in \phi^{-1}(z_{i-1})$ that have $G(s_{i-1}|z_{i-2}, \opt_{i-2}, z_{i-1}) > 0$.  Then, $\phi$ must be dynamics preserving.

\end{proof}

\begin{corollary}
    Let the tuple $(M, \bar{M}, G)$ be a grounded abstract model. 
    Let the strong subgoal property \citep{konidaris2018skills} for an option $\opt$ be defined as, $\text{Pr}(s'|s, \opt) = \text{Pr}(s'|\opt)$.
    The dynamics preserving property holds with a finite abstract state space $\mathcal{Z} = [N]$ for some $N \in \mathbb{N}$  if and only if the strong subgoal property holds.
\end{corollary}
\begin{proof}
    If the strong subgoal property holds, we have that $Pr(s'|s, \opt) = Pr(s' | \opt)$. Then, for any function $\phi: S\rightarrow \mathcal{Z}$, it holds that $\Prob(s'|\phi(s), \opt) = \Prob(s'|s, \opt)$.
    
    Therefore, it is only important to be able to know if a given option is executable in a given abstract state. Therefore, we can construct the function $I_\mathcal{O}(s) = [I_0(s), ..., I_{|\mathcal{O}|}(s)]$ that returns a binary vector that indicates which options are executable in $s$.

    Define the equivalence relation $ s_0 \sim_{\mathcal{O}} s_1$ iff $I_\mathcal{O}(s_1)=I_\mathcal{O}(s_2)$. We can define the abstract state space as $Z \triangleq S / \sim_\mathcal{O}$, that is, the set of equivalent classes. Given that there at most $2^{|\mathcal{O}|} \in \mathbb{N}$ classes, then the abstract MDP is finite.

    We assume that the dynamics preserving property holds and that the abstract state space $Z$ is finite to prove the converse. Then, there exists $\phi : S\rightarrow \mathcal{Z}$ such that $\Prob(s'|\phi(s), \opt) = \Prob(s'|s, \opt)$ and $\Prob(I_\opt=1|s) = \Prob(I_\opt =1|\phi(s))$.

    We can construct a factored $\phi(s) = [\phi_D(s), \phi_I(s)]$, such that, $\Prob(s'|\phi(s), \opt) = \Prob(s'|\phi_D(s), \opt)$ and $\Prob(I_\opt=1|\phi(s)) = \Prob(I_\opt =1|\phi_I(s))$. 

    If we define $\phi_I$ based on the function $I_\mathcal{O}$, as before, then $\phi_I$ maps to a set of at most $2^{|\mathcal{O}|}$ elements. As $\mathcal{Z} = \mathcal{Z}_D\times \mathcal{Z}_I$ is finite, then $\mathcal{Z}_D$ is also finite. Thus, we construct $Z_D = [M]$ and for each option $\opt$ and equivalence class $m\in [M]$ options from each option $\opt$ such that $Pr(s'|\opt_m)  \triangleq Pr(s'|m, \opt)$. Then, the strong subgoal property holds for every $\opt_m$.

\end{proof}

\begin{proposition}
Let $\phi$ be a dynamics-preserving abstraction and $\bar{s} = (\hat{z}, \hat{\opt}, z)$. For $\epsilon > 0$, if $\lVert G_z(s) - G_{\bar{s}}(s)\rVert_1^2 \leq \epsilon$, then there exists $\epsilon_T > 0$ and $\epsilon_R$ > 0 such that $\lVert T(s'|s, \opt) - \tilde{T}(s'|z,\opt) \rVert_1^2 \leq \epsilon_T$ and  $\lVert R(s, \opt) - \tilde{R}(z, \opt) \rVert_1^2 \leq \epsilon_R$.
\end{proposition}

\begin{proof}
    First, we prove that the bounded grounding error implies bounded transition distribution error. If $\phi$ is a dynamics abstraction, then we can learn $\tilde{T}(z'|z, \opt)$ and we have that $T(s'|s, \opt) = T(s'|z, \opt) = \int G_{\bar{s}}(s)\bar{T}(z'|z,\opt)dz'$ and its corresponding approximation $\tilde{T}(s'|z, o) = \int G_{z'}(s)\bar{T}(z'|z,\opt)dz'$

\begin{align*}
    \lVert T(s'|s, \opt) - \tilde{T}(s'|z, \opt) \rVert_1 &= \left| \int \left(G_{\bar{s}}'(s)\bar{T}(z'|z,\opt) - G_{z'}(s)\bar{T}(z'|z,\opt)\right)dz' \right|\\
    &\leq \int \bar{T}(z'|z,\opt)|G_{\bar{s}'}(s)-G_{z'}(s)|dz'ds\\
    &\leq \sqrt{\epsilon}
\end{align*}

Analogously, we can bound the error of the reward function.

\begin{align*}
    \lVert \bar{R}(z',o) - \tilde{R}(z',o) \rVert_1 &= \left| \int G_{\bar{s}'}(s)R(s,\opt)ds -\int G_{z'}(s)R(s,\opt)ds \right|\\
    &\leq \int \left|G_{\bar{s}'}(s)-G_{z'}(s) \right|\left|R(s, \opt)\right| ds\\
    &\leq RMax\int \left|G_{\bar{s}'}(s)-G_{z'}(s) \right|ds\\
    &\leq RMax \sqrt{\epsilon}
\end{align*}

Then, it follows from Minkowski's inequality that 
\begin{align*}
    \lVert R(s,\opt) - \bar{R}(z',\opt) \rVert_1 &= \lVert R(s,\opt) - \tilde{R}(z', \opt) + \tilde{R}(z', \opt) - \bar{R}(z',o) \rVert_1\\
    &\leq \lVert R(s,\opt) - \tilde{R}(z', \opt) \rVert_1 + \lVert \tilde{R}(z', \opt) - \bar{R}(z',o) \rVert_1\\
    &\leq \sqrt{\epsilon} + RMax\sqrt{\epsilon} = \sqrt{\epsilon_R}
\end{align*}
\end{proof}
\begin{theorem}[Value Loss Bound]
    Let $(M, \bar{M}, G)$ be a grounded abstract model and $\tilde{T}(s'|\bar{s}, \opt) = \int G_{\bar{s}'}(s') \bar{T}(\bar{s}'|\bar{s}, \opt)d\bar{s}'$ be the approximate transition dynamics from the grounded model. If the following conditions hold for all $\opt \in \mathcal{O}$ and all $s\in \St$ with $G_{\bar{s}}(s) > 0$: (1)
        $\lVert T(s'|s, \opt) - \tilde{T}(s'|\bar{s},\opt)\rVert_1^2 \leq \epsilon_T$, and 
        (2)$\lvert R(s, \opt) - \bar{R}(\bar{s}, \opt)\rvert^2 \leq \epsilon_R$;
    then, for any policy $\pi$,
    \begin{equation*}
        \lvert Q^\pi(s, \opt) - Q^\pi(\bar{s}, \opt)\rvert \leq \frac{\sqrt{\epsilon_R} + \gamma VMax\sqrt{\epsilon_T}}{1-\gamma}.
    \end{equation*}
\end{theorem}
\begin{proof}
We proceed by induction on $Q_n^\pi(\bar{s},o)$, where
\begin{align}
    %
    v_0^{\pi}(\bar{s}) &= \E_{s \sim \bar{s}}\left[v^\pi(s)\right], \\
    %
    Q_1^\pi(\bar{s},o) &= \int_{s \in \St} P(s)\left(R(s,o)  +\gamma^\tau v_0^\pi(\bar{s}')\right)ds, \\
    &= \int_{s \in \St} P(s) \left(R(s,o) + \gamma^\tau \int_{s' \in \St} T^{s,o,s'}v^\pi(s') ds'\right)ds, \\
    %
    %
    %
    Q_i^\pi(\bar{s},o) &= \int_{s \in \St} P(s) \left(R(s,o) + \gamma^\tau v_{i-1}^\pi(\bar{s}')\right) ds,
\end{align}
with $\bar{s}' = T(\cdot \mid s,o)$. I use $P(s)$ as shorthand for $P(s \sim \bar{s})$ and $T^{s,o,s'}$ for $T(s' \mid s,o)$, and let
\begin{equation}
    \eps_{Q,n} = \sum_{i=0}^{n} \sqrt{\eps_R} + \gamma^i \left(\textsc{VMax}\sqrt{\eps_T}\right).
\end{equation}

\begin{proof}[Base Case: $Q^\pi \approx Q_{1}^{\pi}$]
    \begin{align}
        &Q^\pi(s,o) - Q_{1}^{\pi}(\bar{s},o) \\
        %
        &= R(s,o) + \gamma^\tau \int_{s'} T^{s,o,s'} v^\pi(s') ds' - \int_s P(s) \left(R(\bar{s},o) - \gamma^\tau v_0^\pi(\bar{s}')ds\right), \\
        %
        &= \underbrace{R(s,o) - R(\bar{s},o)}_{\leq \sqrt{\eps_R}}  + \gamma^\tau \int_{s'} T^{s,o,s'} v^\pi(s') ds' - \int_s P(s) \gamma^\tau v_0^\pi(\bar{s}')ds, \\
        %
        &\leq \sqrt{\eps_R}  + \gamma^\tau \int_{s'} T^{s,o,s'} v^\pi(s') ds' - \gamma^\tau \int_s P(s) \mathbb{E}_{s' \sim \bar{s}'}[v^{\pi}(s')] ds \\
        %
        &\leq \sqrt{\eps_R}  + \gamma^\tau \int_{s'} T^{s,o,s'} v^\pi(s') ds' - \gamma^\tau \int_s P(s) \int_{s'}P(s' \sim \bar{s}')v^\pi(s') ds'\ ds, \\
        %
        &\leq \sqrt{\eps_R}  + \gamma^\tau \int_{s'} T^{s,o,s'} v^\pi(s') ds' - \gamma^\tau \int_s P(s) \int_{s'}T^{s,o,s'}v^\pi(s') ds'\ ds, \\
        &\leq \sqrt{\eps_R}  + \gamma^\tau \textsc{VMax} \underbrace{\int_{s'} T^{s,o,s'} - \int_s P(s) T^{s,o,s'} ds\ ds'}_{\leq \sqrt{\eps_T}}, \\
        %
        &\leq \sqrt{\eps_R}  + \gamma^\tau \textsc{VMax}\sqrt{\eps_T}.
    \end{align}
This concludes the base case.
\end{proof}

\begin{proof}[Inductive Case: $Q^\pi \approx Q_{n}^{\pi} \implies Q^\pi \approx Q_{n+1}^{\pi}$] 

\noindent We assume that, for every $s \in \St$ and any $o$,
\begin{equation}
    Q^\pi(s,o) - Q_{n}^{\pi}(\bar{s},o) \leq \eps_{Q,n},
\end{equation}
and prove that
\begin{equation}
    Q^\pi(s,o) - Q_{n+1}^{*}(\bar{s},o) \leq \eps_{Q,n+1}.
\end{equation}

\noindent By algebra,
\begin{align}
    &Q^\pi(s,o) - Q_{n+1}^{\pi}(\bar{s},o)\\
    %
    &= R(s,o) + \gamma^\tau \int_{s'} T^{s,o,s'} v^\pi(s') ds' - \int_{s} P(s) \left(R(s,o) + \gamma^\tau v_n^\pi(\bar{s}')\right) ds,\\
    %
    &= \underbrace{R(s,o) - R(\bar{s},o)}_{\leq \sqrt{\eps_R}}  + \gamma^\tau \int_{s'} T^{s,o,s'} v^\pi(s') ds' - \gamma^\tau \int_s P(s) v_n^\pi(\bar{s}') ds,\\
    %
    &\leq \sqrt{\eps_R}  + \gamma^\tau \int_{s'} T^{s,o,s'} v^\pi(s') ds' - \gamma^\tau \int_s P(s) v_n^\pi(\bar{s}') ds,\\
    %
    &= \sqrt{\eps_R}  + \gamma^\tau \int_{s'} T^{s,o,s'} v^\pi(s') ds' - \gamma^\tau \int_s P(s) \underbrace{v_n^\pi(\bar{s}')}_{\geq \mathbb{E}_{s' \sim \bar{s}'}[v^{\pi}(s')] - \eps_{Q,n}}ds,
\end{align}
\begin{align}
    %
    &\leq \sqrt{\eps_R}  + \gamma^\tau \int_{s'} T^{s,o,s'} v^\pi(s') ds' - \gamma^\tau \int_s P(s) \left(\mathbb{E}_{s' \sim \bar{s}'}[v^{\pi}(s')] - \eps_{Q,n}\right) ds, \\
    %
    &= \sqrt{\eps_R}  + \gamma^\tau \int_{s'} T^{s,o,s'} v^\pi(s') ds' - \gamma^\tau \int_s P(s) \int_{s'} T^{s,o,s'} v^\pi(s') ds'\ ds + \gamma^\tau \eps_{Q,n}, \\
    %
    &= \sqrt{\eps_R}  + \gamma^\tau \int_{s'} T^{s,o,s'} v^\pi(s') ds' - \gamma^\tau \int_{s'}\underbrace{\int_s P(s)  T^{s,o,s'}}_{=T^{\bar{s},o,s'}} v^\pi(s') ds\ ds' + \gamma^\tau \eps_{Q,n}, \\
    %
     &\leq \sqrt{\eps_R}  + \gamma^\tau \textsc{VMax} \underbrace{\int_{s'} T^{s,o,s'}  - T^{\bar{s},o,s'}  ds'}_{\leq \sqrt{\eps_T}} + \gamma^\tau \eps_{Q,n}, \\
    &\leq \sqrt{\eps_R}  + \gamma^\tau \textsc{VMax} \sqrt{\eps_T} + \gamma^\tau \eps_{Q,n}, \\
    %
    &\leq \sqrt{\eps_R}  + \gamma \textsc{VMax} \sqrt{\eps_T} + \gamma\eps_{Q,n}, \\
    %
    &= \eps_{Q,n+1}.
\end{align}
This concludes the inductive case.
\end{proof}
Thus, by induction and the convergence of the geometric series, for any $s, o, \pi$, we conclude that
\begin{equation}
    Q^{\pi}(s,o) - Q^{\pi}(\bar{s},o) \leq \frac{\sqrt{\eps_R} + \gamma \textsc{VMax} \sqrt{\eps_T}}{1-\gamma}.
\end{equation}
\end{proof}

\subsection{TPC is Dynamics Preserving}\label{appendix:tpc-derivation}

We start by considering that by learning an abstract state space such that $MI(S'; Z, O)$ is maximized. The following decomposition based on the mutual information chain rule corresponds to the TPC algorithm \citep{nguyen2021temporal}. In the original paper, they work at the primitive action level and all actions available always, hence, there's no need to consider initiation sets.

\begin{align*}
    MI(S', Z'; Z, O) &\overset{(a)}{=} MI(S'; Z, O) + \underbrace{MI(Z';Z, O|S')}_{=0}; \\
    &\overset{(b)}{=} MI(Z';Z, O) + \underbrace{MI(S'; Z, O | Z')}_{(1)}; \\
    &\overset{(c)}{=} MI(Z';Z, O) + MI(S'; Z,A) -MI(S';Z') + MI(S';Z'|Z,O);
\end{align*}

where (a) follows from the fact that give $s'$ we can determine $z'$, (b) follows from decomposing the term on the left-hand size and (c) from decomposing term (1).

The above implies that $MI(Z';Z,O) = MI(S';Z') - MI(S';Z'|Z,O)$. Therefore, if we maximize both sides of this identity, we must have a latent space that preserve \textit{only} the information of the state $s'$ that is predictable from the previous $(z,a)$ pair. $MI(Z'; Z, O)$ ensures that the next abstract state is predictable from the $(z,o)$ tuple. $MI(S;Z)$ ensures that the abstract state has information about the ground state which is measured by $g(s|z)$.

\begin{equation}
    MI(S;O) = \int p(s,z) \log\frac{g(s|z)}{p(s)}dsdz
\end{equation}

The following decomposition shows the two extra terms required by the TPC algorithm to estabilize the optimization. Term $(a)$ is the (differential) entropy of $\phi$ which tends to infinity for a deterministic function. This is solved by smoothing it with Gaussian noise of $0$ mean and fixed standard deviation, as done in TPC. The second term $(b)$ corresponds to the consistency term, that is, the transition function $p(z'|z,a)$ must have low entropy, which ensures that the abstract dynamics are learnt.

\begin{align*}
    M(S';Z'|Z, O) &= \int p(s',z',z, o) \log\frac{p(s', z'|z, o)}{p(s'|z,o)p(z'|z,o)}ds'dz'dzdo\\
    &= \int p(s',z',z,o) \log\frac{p(z'|s')}{p(z'|z,o)}\\
    &= \underbrace{\int p(s',z')\log p(z'|s')ds'dz'}_{(a)} - \underbrace{\int p(z', z, o) \log p(z'|z,o)dz'dzdo}_{(b)}
\end{align*}

By maximizing $MI(Z';Z, O)$ and $MI(S';Z')$ using InfoNCE \citep{oord2018cpc}, we obtain the TPC algorithm.

\section{Experiments}
For all our planning experiments we use DDQN \citep{van2016deep} modified to consider initiation sets for action selection and target computation to make it compatible with options. We use Adam \citep{kingma2014adam} as optimizer. As exploration, we use linearly decaying $\epsilon$-greedy exploration.

\subsection{Experiments} \label{appendix:experiment-details}

\subsubsection{Environments} 

\begin{figure}
    \centering
    \includegraphics[width=0.7\textwidth]{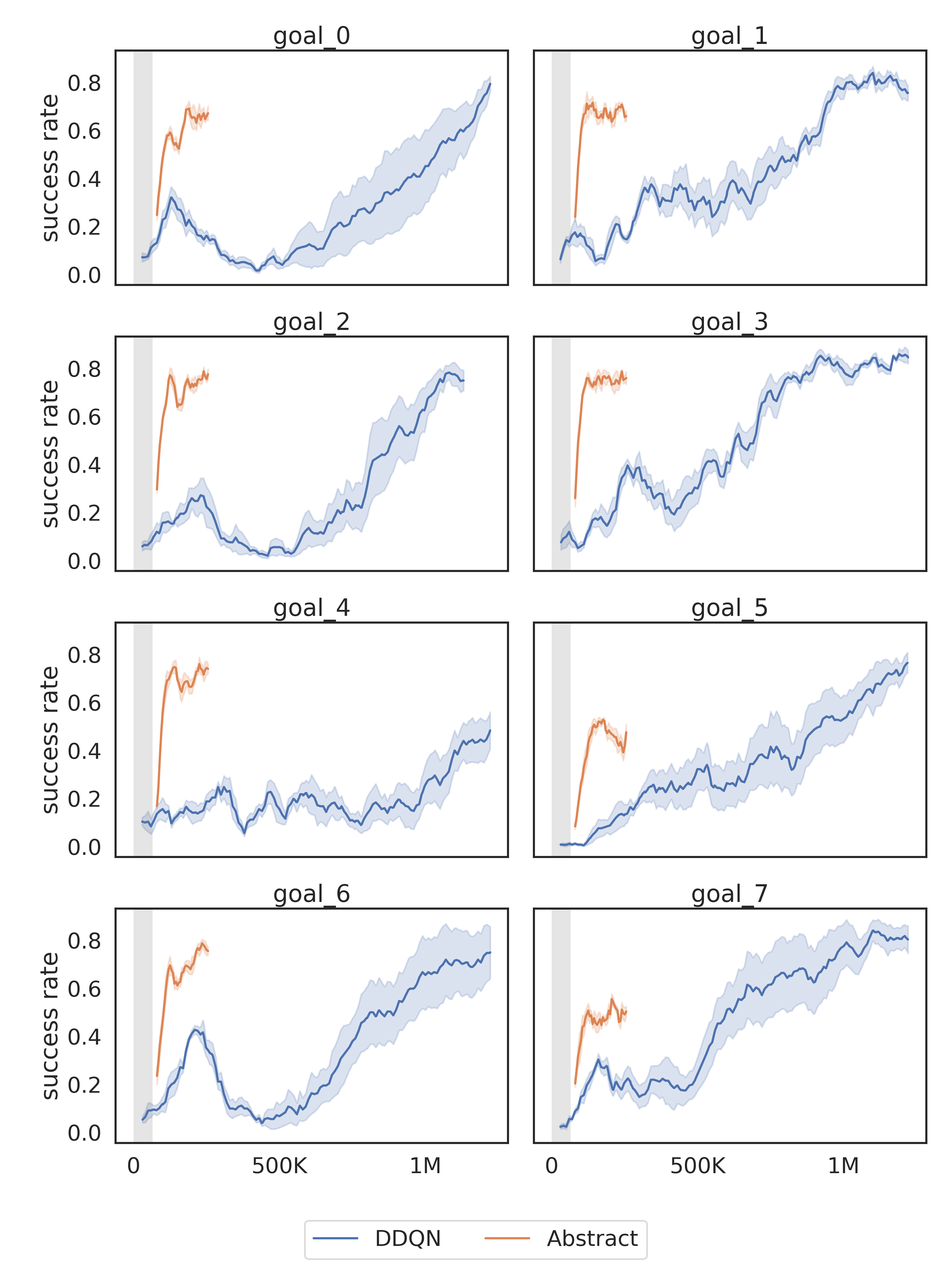}
    \caption{\textbf{Pinball from pixels}. Ground baseline vs Abstract planning. Each goal learning curve is averaged over $5$ seeds and $1$ standard deviation shown in the shaded area of each curve. The gray area corresponds to the offset that corresponds to samples used to pre-train the model. Although is shown in every plot, it is common to all goals.}
    \label{}
\end{figure}



\begin{figure}
    \centering
    \includegraphics[width=.6\textwidth]{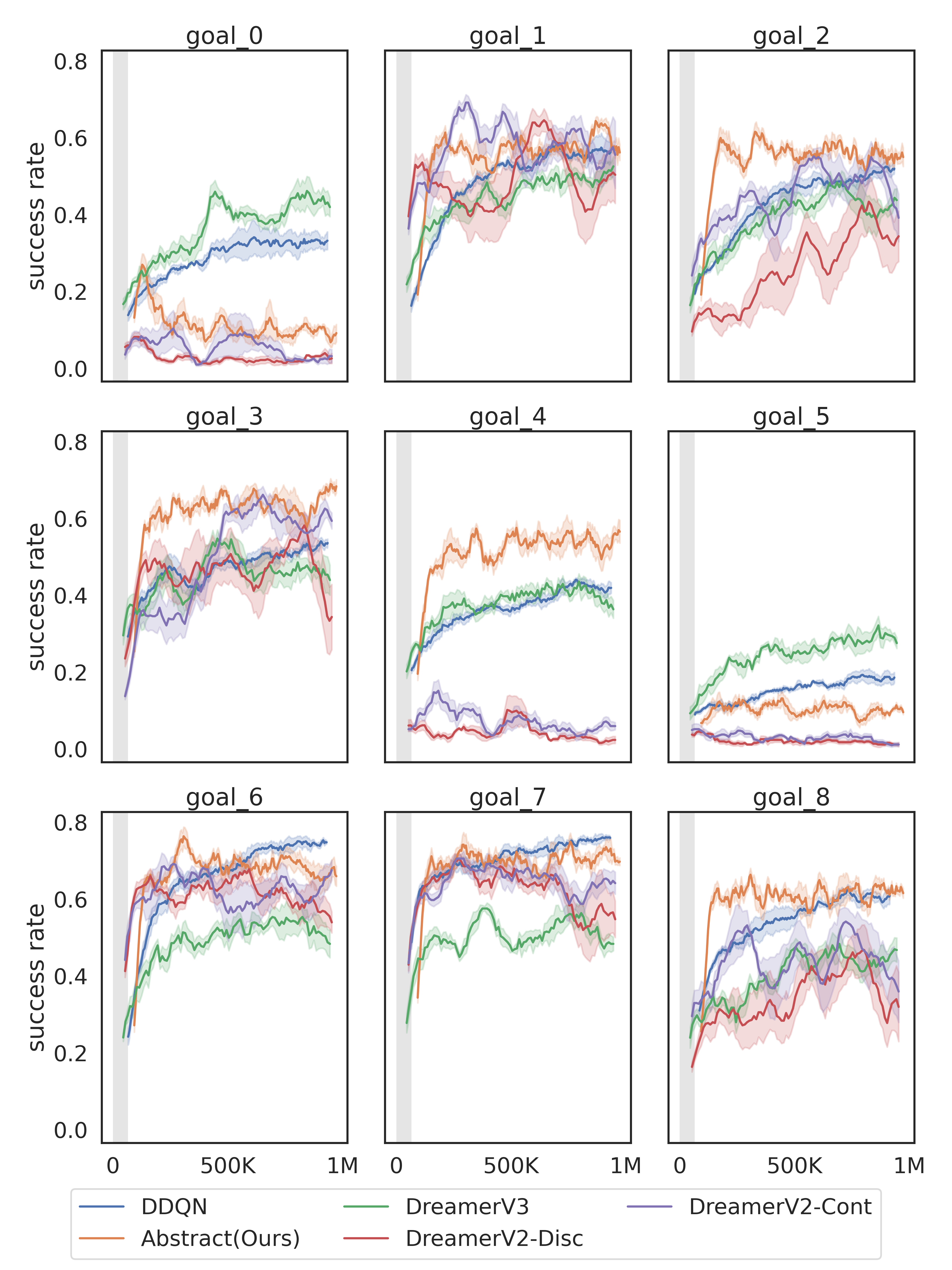}
    \caption{\textbf{Medium Play Antmaze}. Ground baseline vs Abstract planning. Each goal learning curve is averaged over $5$ seeds and $1$ standard deviation shown in the shaded area of each curve. The gray area corresponds to the offset that corresponds to samples used to pre-train the model. Although is shown in every plot, it is common to all goals.}
    \label{}
\end{figure}

\begin{figure}
    \centering
    \includegraphics[width=0.3\textwidth]{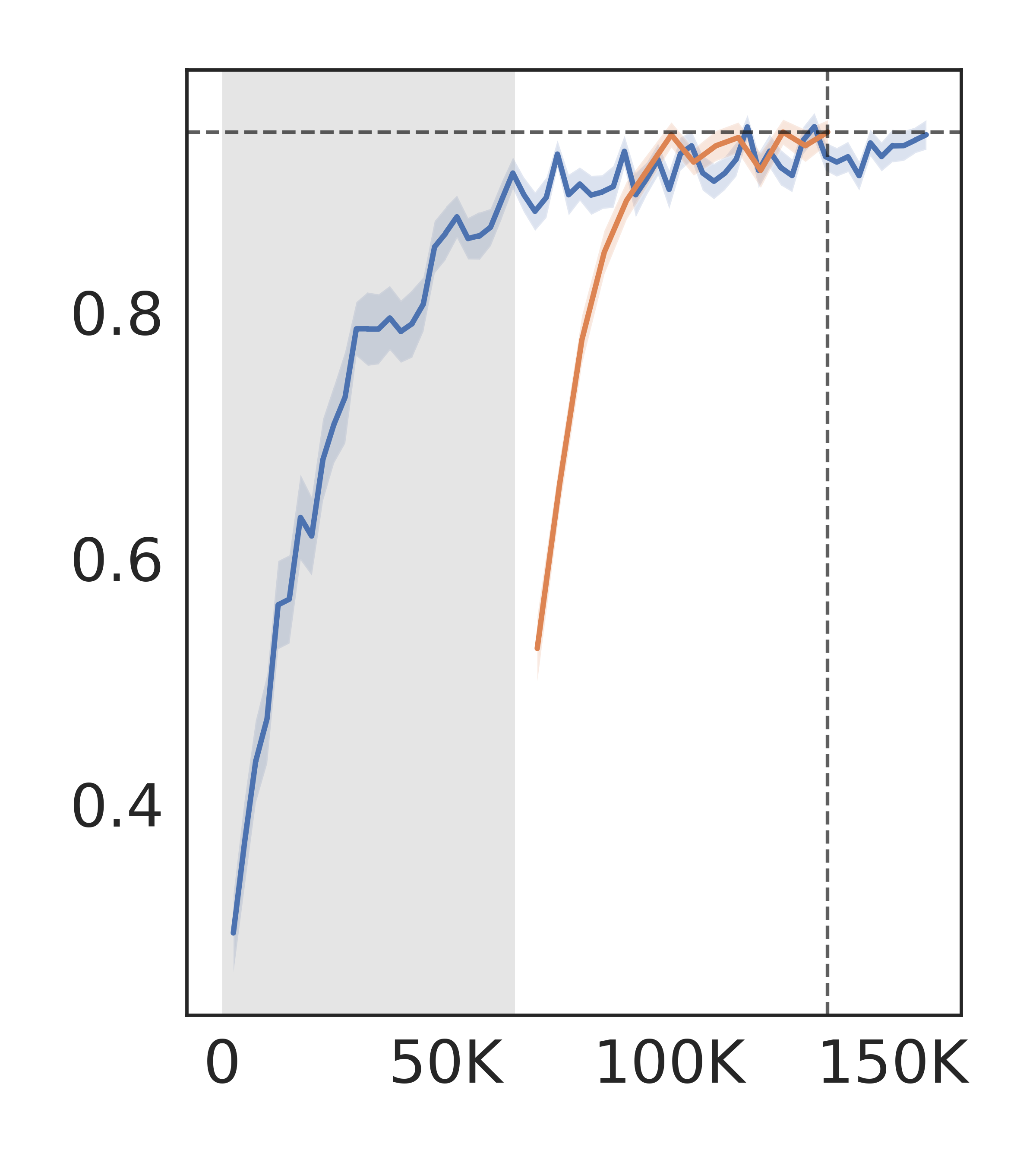}
    \includegraphics[width=0.3\textwidth]{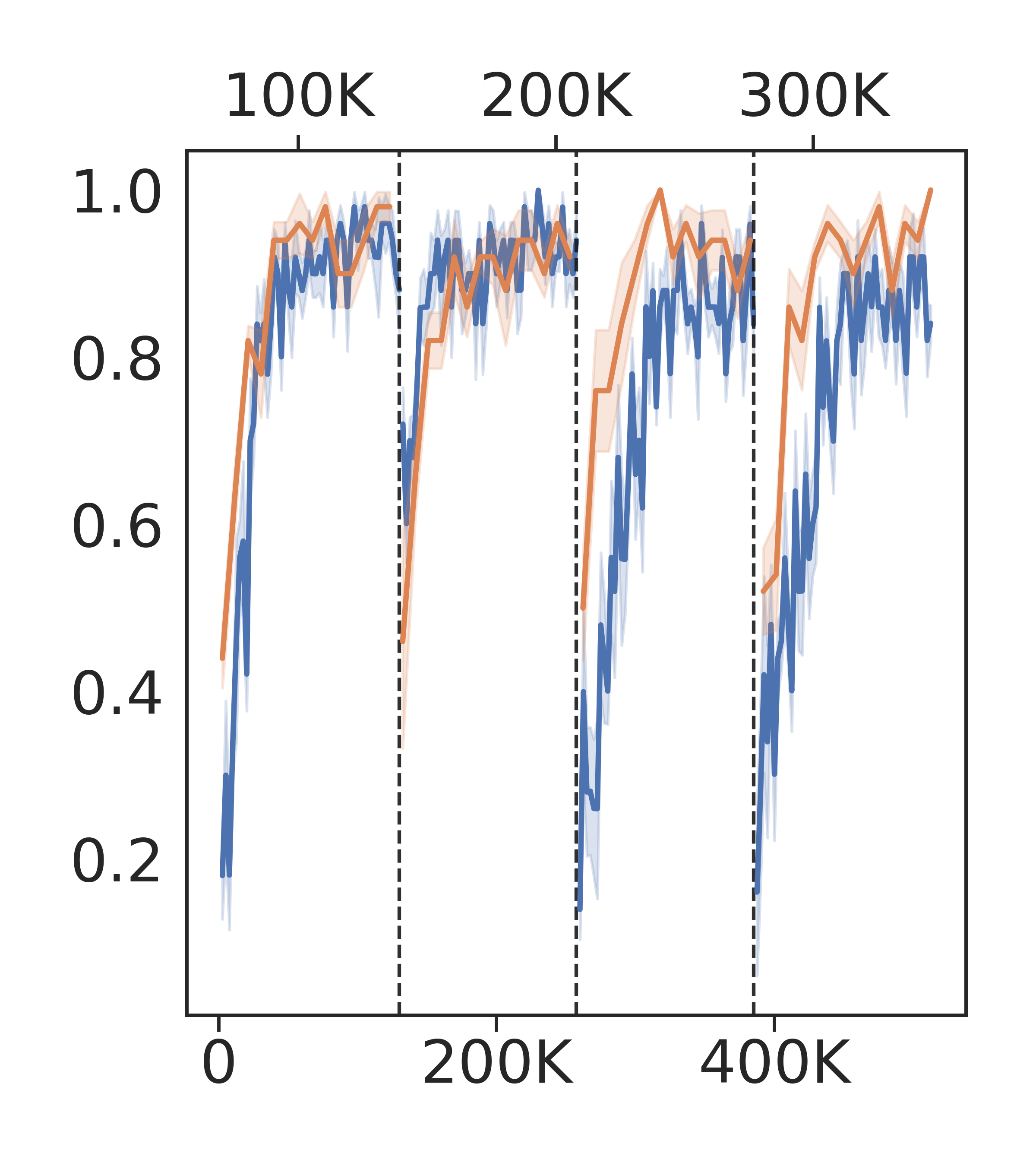}\\
    \includegraphics[width=0.15\textwidth]{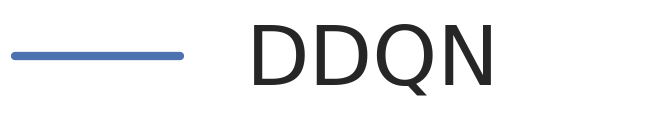}
    \includegraphics[width=0.2\textwidth]{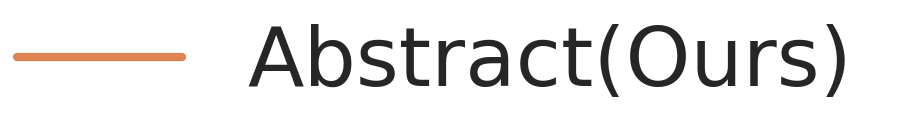}
    \caption{\textbf{U-Maze Antmaze}.}
    \label{}
\end{figure}

\begin{figure}
    \centering
    \includegraphics[width=0.33\textwidth]{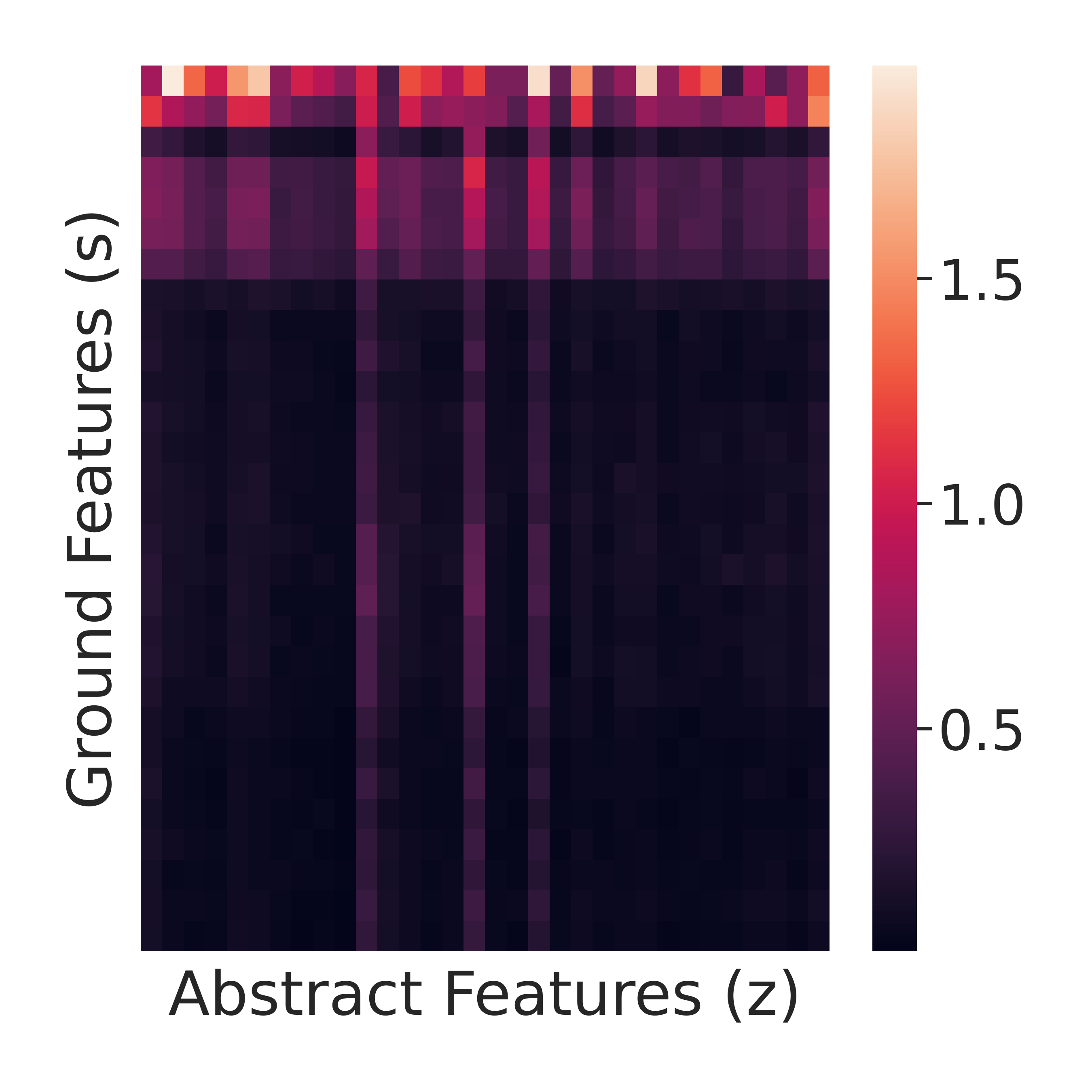}
    \includegraphics[width=0.3\textwidth]{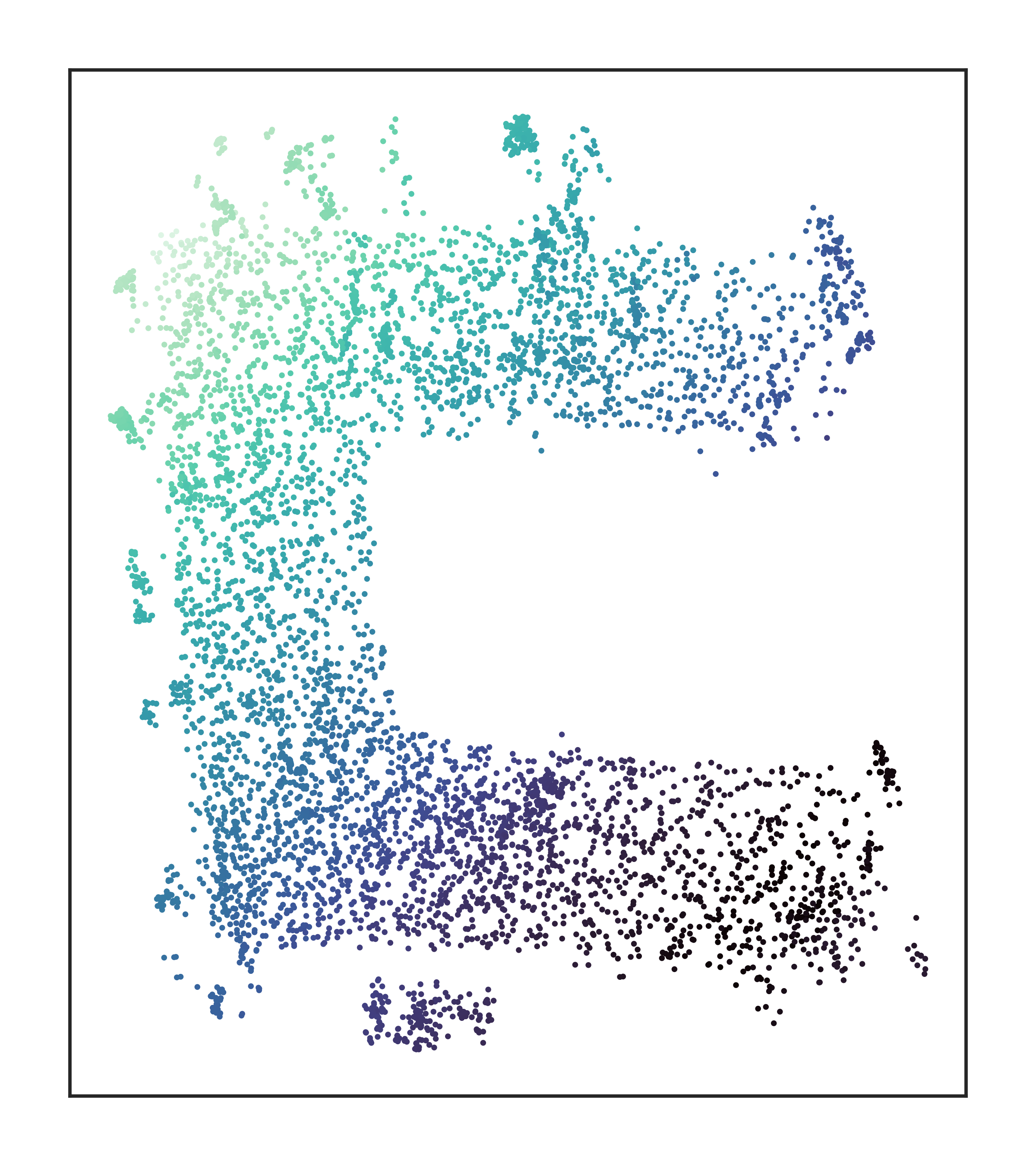}
    \caption{\textbf{U-Maze Antmaze}.}
    \label{}
\end{figure}

\paragraph{Pinball Domain \citep{KonidarisSkillChaining09}} We use a continuous action variant of the original environment. The state space $s = (x,y,\dot{x},\dot{y})$ with $(x,y) \in [0,1]^2$ and $(\dot{x}, \dot{y}) \in [-1,1]$. The action space is the ball acceleration expressed in the form of $\Delta(\dot{x}, \dot{y}\in [-1, 1]^2$. The layout of the obstacles is as in the original environment, show in Figure \ref{}. The reward function takes $-5$ per unit of acceleration. The discount factor is $\gamma=0.9997$.

\paragraph{Pinball Options} Pinball options were designed to the agent in the coordinate dimensions by step size $0.04$. The initiation set are all the position in which the ball would not hit an obstacle by moving in the desired direction. The termination probability is determined by a Gaussian centered in the goal position with standard deviation as $0.01$. For the policy, we handcrafted PI controllers for the position with constants $K_p=50$ and $K_i=8$.

\paragraph{Antmazes} We consider the U-Maze and Medium-Play mazes implemented by D4RL \citep{fu2020d4rl} with the Mujoco ant. In Figure \ref{fig:antmazes} we show diagrams of the considered mazes. The state space is $\St\in\mathbb{R}^{29}$, where the first two dimensions corresponds to the position of the ant in the maze and the rest is proprioception for the ant controls. The action space is $\A\subset[-1,1]^8$ to control the ant joints.

\begin{figure}
  \centering
  
  \begin{subfigure}[b]{0.3\textwidth}
    \centering
    \includegraphics[width=\textwidth]
    {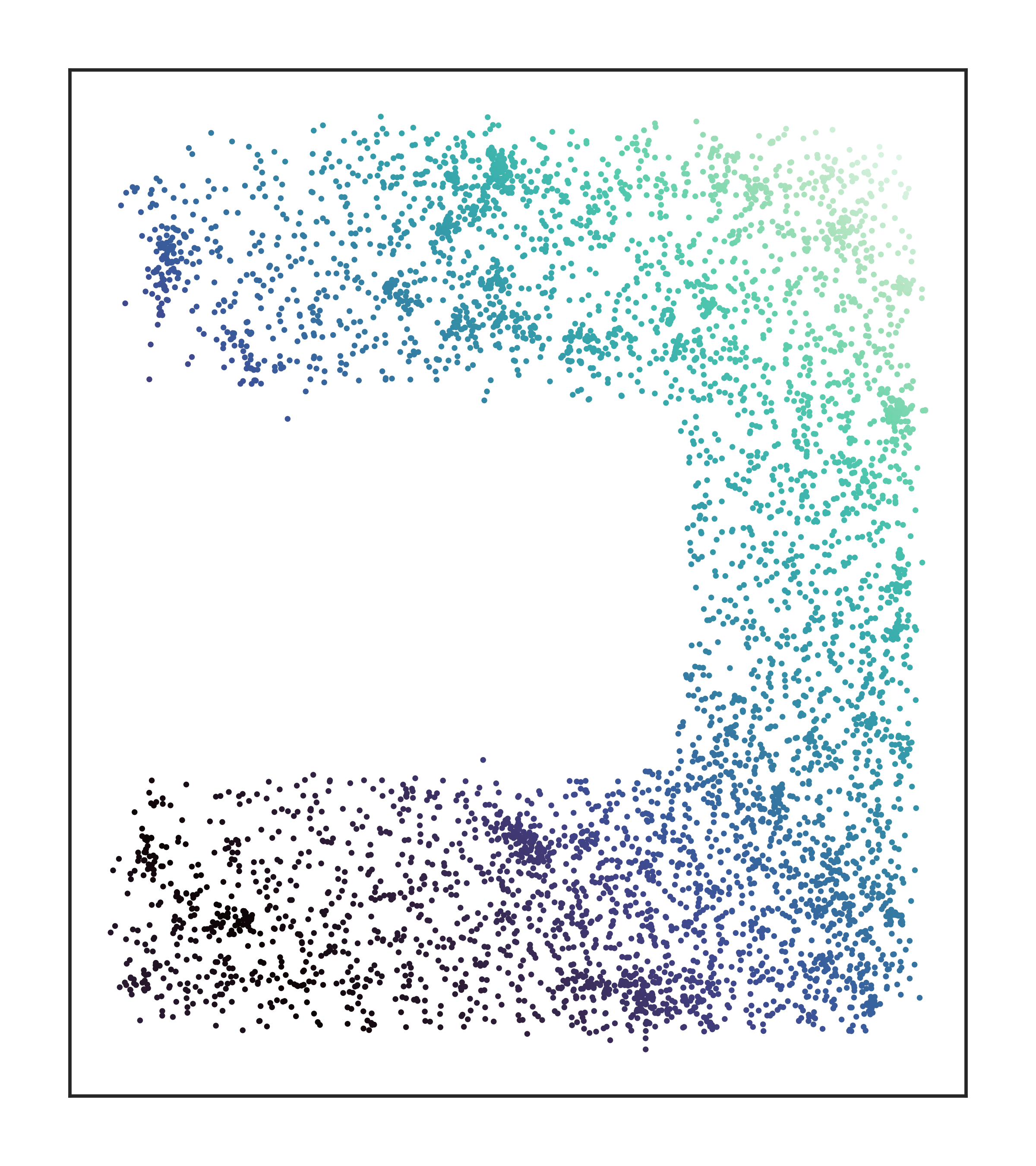}
    \caption{U-Maze}
  \end{subfigure}
  \begin{subfigure}[b]{0.3\textwidth}
    \centering
    \includegraphics[width=\textwidth]
    {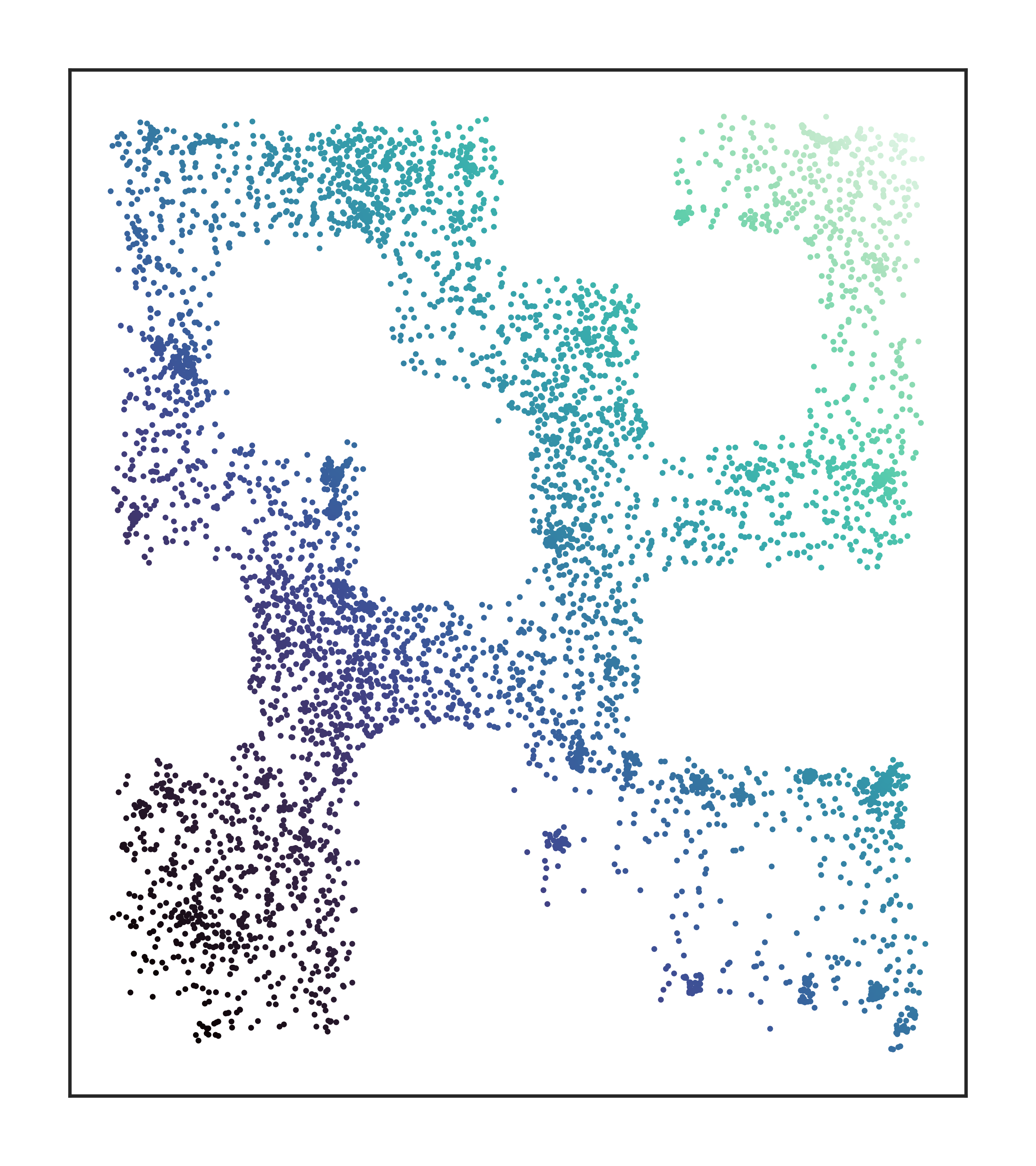}
    \caption{Medium Play Maze}
  \end{subfigure}
  \begin{subfigure}[b]{0.3\textwidth}
    \centering
    \includegraphics[width=\textwidth]
    {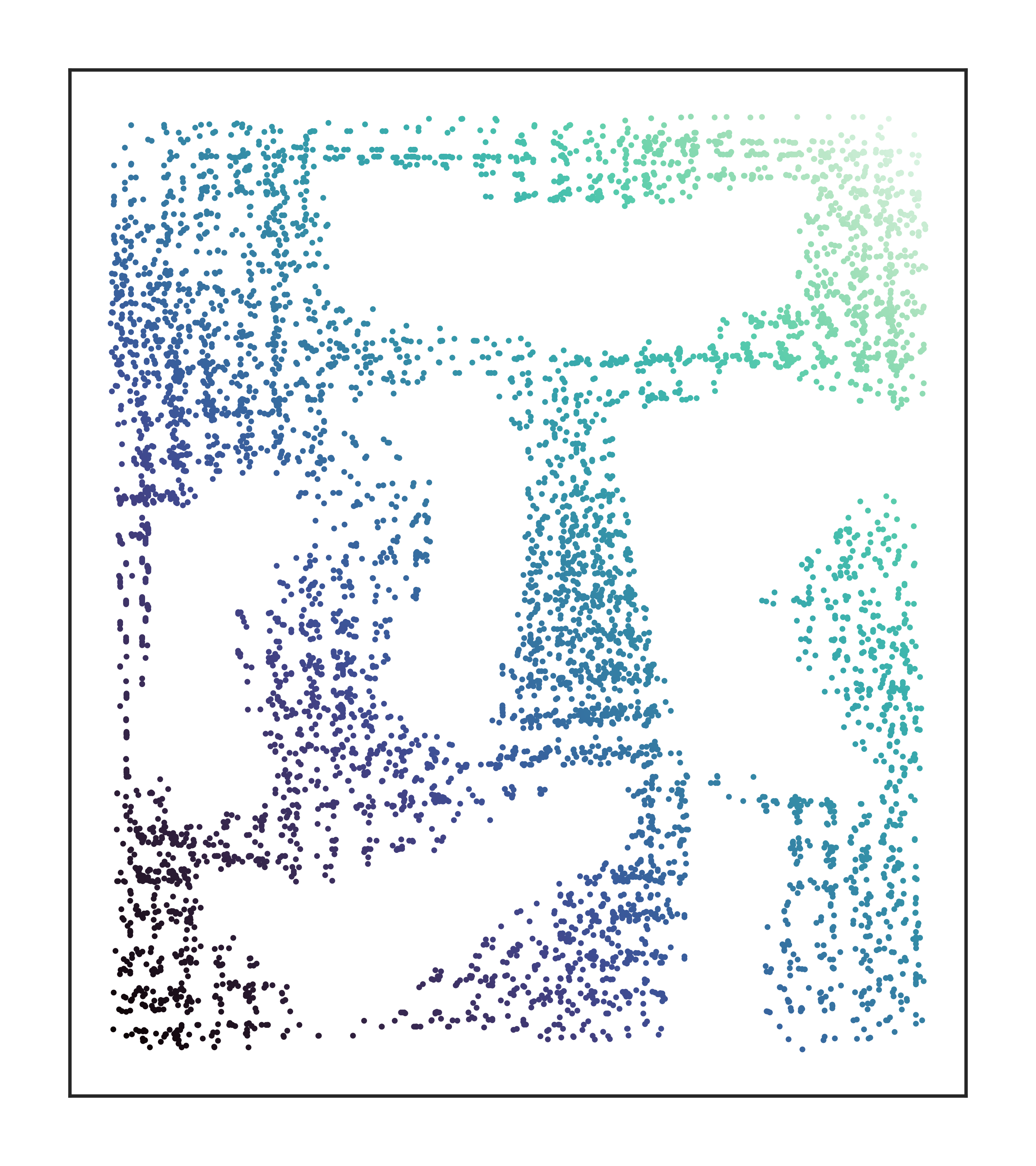}
    \caption{Pinball Domain}
  \end{subfigure}
  \caption{Ground truth visualization of possible positions of the agent in the evaluation Environments}\label{fig:antmazes}
\end{figure}

\paragraph{Antmaze Options} We consider options that move the ant in the $8$ directions (North, South, East, West, North-East, North-West, South-East, South-West) by a distance of $1$ unit. For the position controller, we train a goal-conditioned policy using Hindsight Experience Replay (HER; \cite{andrychowicz2017hindsight}) and TD3 \citep{fujimoto2018addressing} that would take a goal position in an drive walk the ant to it. This is generally hard for arbitrary goals given the separation between the current position and the goal, however, we only needed the policy to become accurate for short distances, so we sampled initial positions within $1.5$ of the desired goal. The goals were sampled uniformly over the possible positions in the maze. Then we learned the initiation sets as classifiers were the option execution would be successful. The termination condition is a threshold of $0.5$ distance to the goal.

\subsubsection{Network Architectures}

\paragraph{Pixel Observations} As encoder for pixel observation, we use ResNet Convolutional Networks, as used in Dreamer \citep{hafner2020mastering}. The ResNet starts with an initial $24$ depth and doubles in depth until reaching the minimal resolution. See Table \ref{table:cnn-conf}.

\begin{table}[h]
\centering
\caption{ResNet CNN Configuration} \label{table:cnn-conf}
\begin{tabular}{ll}
\toprule
Parameter & Value \\
\midrule
in width & $50$ \\
in height & $50$ \\
color channels & $1$ \\
depth & $24$ \\
cnn blocks & $2$ \\
min resolution & $4$ \\
mlp layers & $[256, ]$ \\
outdim & $4$ \\
mlp activation & silu \\
cnn activation & silu \\
\bottomrule
\end{tabular}
\label{tab:nn_configuration}
\end{table}

\paragraph{MLP Architectures} For all other models, we use MLPs  with the relevant input and output dimensions. This includes encoder, initiation classifiers, transition function, reward function and duration. For the reward function we use the symlog transformation \citep{hafner2020mastering} and a log transformation for the option duration network.
\begin{table}[h]
\centering
\caption{MLP Configuration}
\begin{tabular}{ll}
\toprule
Parameter & Value \\
\midrule
hidden dims & $[128, 128]$ \\
activation & relu \\
\bottomrule
\end{tabular}
\label{tab:reward_mlp_configuration}
\end{table}

\paragraph{Density Estimation} We use mixture of Gaussians with $4$ components and Gaussians with diagonal covariance matrices. We use the reparameterization trick \citep{kingma2013auto} to optimize the mean and variance functions. 

\subsubsection{Agent Hyperparameters}

To train our baseline DDQN agent with the following parameters that we tune by doing grid search for $5$ goal positions and $2$ seeds, we use all these parameters to learn for all goals.

\paragraph{Pinball Domain} For pixel observations we use the same architecture as described before for the world model encoder. For simpler observation, we use an MLP as before.

\begin{table}[h]
\centering
\caption{Pinball ground DDQN parameters}
\begin{tabular}{l|c}
\toprule
\textbf{Parameter} & \textbf{Value} \\
\midrule
final exploration steps & $500000$ \\
final epsilon & $0.1 $\\
eval epsilon &$ 0.001$ \\
replay start size & $10000$ \\
replay buffer size & $500000$ \\
target update interval & $10000 $\\
steps & $1250000$ \\
update interval &$ 5$ \\
num step return & $1$ \\
learning rate & $10^{-5}$ \\
$\gamma$ & $0.9997$\\
\bottomrule
\end{tabular}

\end{table}

\begin{table}[h]
\centering
\caption{Ground DDQN Parameters for the Antmazes}
\begin{tabular}{l|c}
\toprule
\textbf{Parameter} & \textbf{Value} \\
\midrule
final exploration steps & $350,000$ \\
final epsilon & $0.1$ \\
eval epsilon & $0.001$ \\
replay start size & $1,000$ \\
replay buffer size & $100,000$ \\
target update interval & $1,000$ \\
steps & $1,000,000$ \\
update interval & $5$ \\
num step return & $1$ \\
learning rate & $5 \times 10^{-4}$ \\
$\gamma$ & $0.995$ \\
\bottomrule
\end{tabular}
\end{table}

\begin{table}[h]
\centering
\caption{U-Maze Imagination DDQN Parameters}
\begin{tabular}{l|c}
\toprule
\textbf{Parameter} & \textbf{Value} \\
\midrule
final exploration steps (proportion) & 30\% of agent training steps \\
final epsilon & 0.1 \\
eval epsilon & 0.001 \\
replay start size & 1000 \\
replay buffer size & 100000 \\
target update interval & 10000 \\
update interval & 5 \\
num step return & 1 \\
learning rate & $1 \times 10^{-4}$ \\
rollout length & 100 \\
\bottomrule
\end{tabular}
\end{table}

\paragraph{Dreamer Baselines}
We use the publicly available implementations for the Dreamer baselines.
For the DreamerV2 \citep{hafner2020mastering} baseline, we used the hyperparameters recommended for DeepMind Control environments with (only) proprioception inputs.
Instead, for the DreamerV3 \citep{hafner2023mastering} baseline we used the recommended hyperparameters.

\subsubsection{World Model Hyperparameters}

\begin{table}[h]
\centering
\caption{World Model Parameters}
\begin{tabular}{l|c}
\toprule
\textbf{Parameter} & \textbf{Value} \\
\midrule
buffer size & $100,000$ \\
batch size & $16$ \\
learning rate & $1 \times 10^{-4}$ \\
train every & $8 $\\
max rollout length & $64 $\\
\bottomrule
\end{tabular}
\end{table}

\end{document}

%% file: math_commands.tex
\usepackage{amsmath,amsfonts,bm}

\def\eqref#1{equation~\ref{#1}}

\def\1{\bm{1}}

\def\eps{{\epsilon}}

\DeclareMathAlphabet{\mathsfit}{\encodingdefault}{\sfdefault}{m}{sl}
\SetMathAlphabet{\mathsfit}{bold}{\encodingdefault}{\sfdefault}{bx}{n}

\newcommand{\E}{\mathbb{E}}

\newcommand{\R}{\mathbb{R}}

